%% file: camera-ready.tex
\documentclass{article}
\usepackage{kr}

\usepackage{times}
\usepackage{soul}
\usepackage{url}
\usepackage[hidelinks]{hyperref}
\usepackage[utf8]{inputenc}
\usepackage[small]{caption}
\usepackage{graphicx}
\usepackage{amsmath}
\usepackage{amsthm}
\usepackage{booktabs}
\usepackage{algorithm}
\usepackage{algorithmic}
\usepackage{array} 
\usepackage{microtype}
\usepackage{graphicx}
\usepackage{booktabs} % for professional tables
\usepackage{enumitem}
\usepackage{subcaption}

\usepackage{svg}
\usepackage{tikz}
\usetikzlibrary{calc}
\usetikzlibrary{shadings, decorations.pathreplacing}

\usepackage{multirow}

\usepackage{algorithm}
\usepackage{algorithmic}
\usepackage{amsfonts} % add by Hui
\usepackage{amsmath} % add by Hui
\usepackage{amsthm}
\newtheorem{defi}{Definition}

\newtheorem{example}{Example}
\newtheorem{prop}{Proposition}
\newtheorem{theo}{Theorem}

\DeclareMathOperator{\arcsinh}{arcsinh}
\usepackage{amsmath}
\usepackage{bm}
\usepackage{xspace}
\usepackage{xcolor}

\theoremstyle{plain}
\newtheorem{theorem}{Theorem}[section]

\theoremstyle{definition}
\newtheorem{definition}[theorem]{Definition}

\theoremstyle{remark}

\usepackage{hyperref}

\usepackage{amsmath}
\usepackage{amssymb}
\usepackage{mathtools}
\usepackage{amsthm}

\usepackage[capitalize,noabbrev]{cleveref}

\usepackage[utf8]{inputenc} % allow utf-8 input
\usepackage[T1]{fontenc}    % use 8-bit T1 fonts
\usepackage{hyperref}       % hyperlinks
\usepackage{url}            % simple URL typesetting
\usepackage{booktabs}       % professional-quality tables
\usepackage{amsfonts}       % blackboard math symbols
\usepackage{nicefrac}       % compact symbols for 1/2, etc.
\usepackage{microtype}      % microtypography
\usepackage{xcolor}         % colors

\input{math_commands.tex}

\usepackage{hyperref}
\usepackage{url}

\title{RegD: Hierarchical Embeddings via Dissimilarity between Arbitrary Euclidean Regions}

\author{%
Hui Yang$^1$\and
Jiaoyan Chen$^1$ \\
\affiliations
$^1$The University of Manchester\\
\emails
\{hui.yang-2, jiaoyan.chen\}@manchester.ac.uk
}

\begin{document}
\include{marcos}

\maketitle

\begin{abstract}
Hierarchical data is common in many domains like life sciences and e-commerce, and its embeddings often play a critical role. 
While hyperbolic embeddings offer a theoretically grounded approach to representing hierarchies in low-dimensional spaces, current methods often rely on specific geometric constructs as embedding candidates. This reliance limits their generalizability and makes it difficult to integrate with techniques that model semantic relationships beyond pure hierarchies, such as ontology embeddings. 
In this paper, we present RegD, a flexible Euclidean framework that supports the use of arbitrary geometric regions---such as boxes and balls---as embedding representations. Although RegD operates entirely in Euclidean space, we formally prove that it achieves hyperbolic-like expressiveness by incorporating a depth-based dissimilarity between regions, enabling it to emulate key properties of hyperbolic geometry, including exponential growth. 
We establish the faithfulness of our approach. Furthermore, extensive empirical evaluations on diverse real-world datasets demonstrate consistent performance improvements over state-of-the-art methods, highlighting RegD’s potential for broader applications, including ontology embedding tasks that extend beyond hierarchical structures. 
%Code and data are available at \url{https://anonymous.4open.science/r/RegD-F4E3}.
\end{abstract}

\section{Introduction}
Embedding discrete data into low-dimensional vector spaces has become a cornerstone of modern machine learning. In Natural Language Processing (NLP), seminal works such as word2vec \cite{DBLP:journals/corr/abs-1301-3781} and GloVe \cite{DBLP:conf/emnlp/PenningtonSM14} represent words as vectors to capture intricate linguistic relationships. Similarly, knowledge graph embedding methods \cite{DBLP:conf/nips/BordesUGWY13,DBLP:conf/iclr/SunDNT19,DBLP:conf/nips/BalazevicAH19} encode entities and relations as vectors with their semantics concerned to facilitate reasoning and prediction.

Our work focuses on embedding hierarchical data into low-dimensional spaces. Such data represents \emph{partial orders} over sets of elements, denoted as $u \prec v$, where $u$ is a parent of $v$. These partial orders naturally manifest as trees or directed acyclic graphs (DAGs). The ability to effectively embed such hierarchical structures enables crucial operations like inferring sub- or superclasses of concepts and classifying nodes within graphs. These capabilities are essential for various tasks in knowledge management and discovery, particularly towards knowledge bases \cite{DBLP:conf/www/Shi0CW024,DBLP:conf/nips/AbboudCLS20}, ontologies \cite{DBLP:journals/semweb/HeCDHAKS24,DBLP:journals/corr/abs-2406-10964} and taxonomies \cite{DBLP:series/synthesis-dmk/2022Shen}.
%various machine learning tasks, particularly knowledge base completion \cite{DBLP:conf/www/Shi0CW024,DBLP:conf/nips/AbboudCLS20}, ontology alignment \cite{DBLP:journals/semweb/HeCDHAKS24} and ontology embeddings \cite{DBLP:conf/www/JackermeierC024}.

Current methods for embedding hierarchical data can be broadly categorized into two paradigms: \textit{region-based} and \textit{hyperbolic} approaches. 
The region-based approach usually represents entities as geometric regions in the Euclidean space, capturing hierarchical relationships through intuitive region-inclusion. However, these methods often experience degraded performance in low-dimensional settings due to the crowding effect inherent in Euclidean spaces.

In contrast, the hyperbolic approach takes advantage of the unique geometric properties of hyperbolic spaces (\textit{e.g.,} their exponential growth in distance and volume) enabling more effective embeddings of tree-structured data in low dimensions. However, as shown in Table \ref{tab:region_hyperbolic}, existing hyperbolic methods often rely on specialized constructed objects as embedding candidates \cite{EntCone,ShadowCone}, limiting their generalizability to data that encode richer semantics beyond hierarchy.  For example, EntailmentCones~\cite{EntCone} and ShadowCones~\cite{ShadowCone} are not closed under intersection and thus cannot handle conjunction.
%Additionally, training in hyperbolic space is challenging due to precision issues and the complexities of Riemannian optimization \cite{DBLP:conf/icml/MishneW0023}.

% \begin{enumerate}
%     \item  \textbf{Specific construction requirements:} Hyperbolic embeddings rely on hyperbolic spaces and specific constructions, such as EntailmentCones \cite{EntCone} and ShadowCones \cite{ShadowCone}, that are built manually from theoretical analysis or physical inspirations. This may limit their effectiveness in embedding various DAGs, particularly those with less tree-like structure that are not well-suited for hyperbolic representations. It also limits the generalizability of their implementations to tasks beyond purely hierarchical datasets.
    
%     \item \textbf{Computational complexity:} The computational complexity of hyperbolic embeddings exceeds that of Euclidean embeddings, requiring specialized non-linear functions and high-precision tensor operations that often need double-precision tensors. This means they require twice the space for storing embedding vectors of the same dimension and more than twice the computation time.
% \end{enumerate}

In this paper, we propose a flexible framework named RegD for modeling hierarchical data by embedding arbitrary regions in Euclidean spaces. Our framework relies on two novel dissimilarity metrics, \textit{depth dissimilarity} $\ddepth$ and \textit{boundary dissimilarity} $\dboundary$, which combine the strengths of both region-based approaches in Euclidean spaces and hyperbolic methods.
\textbf{The depth dissimilarity} (cf. Section~\ref{sec:dep_dist}) enables our model to achieve embedding expressiveness comparable to that of hyperbolic spaces by incorporating the ``size'' of the regions under consideration (cf. Theorem~\ref{theo:dep_dist} and Propositions~\ref{prop:dep_eq} and~the general version in Section \ref{app:general_reg}). 
%Its computation involves only simple operations, thereby reducing computational complexity and eliminating the need for enhanced numerical precision required by hyperbolic methods. 
\textbf{The boundary dissimilarity} (cf. Section~\ref{sec:dboundary}) improves the representation of set-inclusion relationships among regions in Euclidean spaces. This allows for better identification of shallower and deeper descendants, thereby capturing hierarchical structures more effectively than traditional region-based approaches (cf. Proposition \ref{prop:hyper}). 
%Note that we do not aim to construct strict metric spaces, so these distances need not be non-negative or satisfy the triangle inequality. 
Note that these dissimilarities may be negative or fail to satisfy the triangle inequality, and thus are generally not strict metric distances. 
% Note that our framework does not require these distances to be non-negative or to satisfy the triangle inequality, as our goal is not to construct strict metric spaces.

%Combining these two distances, we designed a closed formula (cf. Section \ref{sec:train}) for determining the partial order of two concepts based on their embedded regions, which is used in both the training and evaluation processes. 

\begin{table*}
    \centering
    \begin{tabular}{llc}
        \toprule
        \textbf{Method} & \textbf{Embeddings of a node $u$} & \textbf{Energy/score function for $u\prec v$} \\
        \hline
        \textbf{Poincaré}~\cite{DBLP:conf/nips/NickelK17} & Points $\bu\in \mathbb{H}^d$ & $(1 + \alpha(\|\bu\| - \|\bv\|))d_k(\bu, \bv)$  \\
        \hline
        \textbf{EntailmentCones}   %\multirow{2}{*}{\textbf{EntailmentCones}~\cite{EntCone}} 
     & 
        Cones with apex $\bu\in \mathbb{H}^d$ and & 
         \multirow{2}{*}{Angle-based function}\\
        %\multirow{2}{*}{$\max\left\{ 0, \angle(\overrightarrow{\bu\bv}^g, \overrightarrow{\textbf{0}\bu}^g) - \theta_\bu \right\}$} \\
        ~~\cite{EntCone}&  angle $\theta_\bu = \arcsin(\frac{(1-\|\bu\|^2)}{\|\bu\|})$ & \\
        \hline
        \textbf{ShadowCones} %\multirow{2}{*}{\textbf{ShadowCones}~\cite{ShadowCone}} 
        & 
        Cones with apex $\bu\in \mathbb{H}^d$ and  & 
        $d_{\mathbb{H}}(\mathbf{u}, \mathbf{v})$ \textbf{or} \\
         %$d_{\mathbb{H}}(\mathbf{u}, \mathbf{v})$ if $H(\mathbf{v}, \mathbf{u}) > 0$, else\\
        ~~\cite{ShadowCone} &  angle $\theta_\bu = \arctan \sinh (\sqrt{k} r)$ & 
        %$\operatorname{arsinh}(t(\bu, \bv))/\sqrt{k} + r$\\
        specific boundary distance\\
        \hline
       \multirow{2}{*}{\textbf{RegD}~(ours)} & arbitrary regions in $\mathbb{R}^d$ (e.g., balls, & 
       \multirow{2}{*}{$\dboundary(\reg_u, \reg_v) + \lambda \cdot \ddepth(\reg_u, \reg_v)$}\\
       & boxes, detailed in Section \ref{app:general_reg}) & \\
        \bottomrule
    \end{tabular}
    \caption{Comparison with hyperbolic methods in $\mathbb{H}^d$ (curvature $-k$, distance $d_k$). $\alpha$, $r$ and $\lambda$ are predefined hyperparameters. Only the umbral half-space model of ShadowCones is shown, with boundary distance defined in Eq.~\ref{eq:cone_dist}.}
    \label{tab:region_hyperbolic}
\end{table*}

Notably, RegD can be applied to arbitrary regions\footnote{Here, “arbitrary regions” does not denote arbitrary subsets of Euclidean space, but rather regular geometric domains of arbitrary shape that can be parameterized and that satisfy the general properties described in Section~\ref{app:general_reg}.}, including common geometric representations such as balls and boxes. This generality enables broad applicability across diverse geometric embeddings for various tasks, extending beyond hierarchy alone data to ontologies that include hierarchies and more complex relationships in Description Logic \cite{DBLP:books/daglib/0041477} (cf. Sections \ref{sec:Ont_inf} and \ref{sec:Ont_pred}). 
Our main contributions are summarized as follows:
\begin{itemize}[leftmargin=*]
\item We present a versatile framework that is able to embed hierarchical data as arbitrary regions in Euclidean space.
\item We offer a rigorous theoretical analysis demonstrating that our framework retains the core embedding benefits of hyperbolic methods.
\item Experiments on diverse real-world datasets demonstrate that our framework consistently outperforms existing approaches on embedding hierarchies and ontologies for reasoning and prediction.
%hierarchical tasks and its potential for broader applications.
\end{itemize}

Due to space limitations, implementation details, additional experimental results, as well as the code and datasets, are provided at the following anonymous link:
\url{https://github.com/HuiYang1997/RegD}.

\section{Preliminaries and Related Works} 
% \subsection{Hierarchy Embeddings}
\subsubsection{Manifold and Hyperbolic Space}  
A $d$-dimensional \emph{manifold} \cite{Lee_2013}, denoted $\mathcal{M}$, is a hypersurface embedded in an $n$-dimensional Euclidean space, $\mathbb{R}^n$, where $n \geq d$, and locally resembles $\mathbb{R}^d$. 

A \emph{Riemannian manifold} $\mathcal{M}$ is a manifold equipped with a Riemannian metric, enabling the definition of a distance function $d_\mathcal{M}(\mathbf{x}, \mathbf{y})$ for $\mathbf{x}, \mathbf{y} \in \mathcal{M}$.  
\textit{Hyperbolic space}, denoted $\mathbb{H}^n$, is a Riemannian manifold with a constant negative curvature of $-k$, where $k > 0$ \cite{lee2006riemannian}. It can be represented using various isometric models, such as   Poincaré half-space model, where the points are defined by the half-space:
\(
U^n = \{ \x \in \mathbb{R}^n : \x_n > 0 \},
\)
and the hyperbolic distance between $\x, \y \in U^n$ is given by
\[
d_k(\x, \y) = \frac{1}{\sqrt{k}} \, \operatorname{arcosh}\left( 1 + \frac{\|\x - \y\|_2^2}{2 \x_n \y_n} \right).
\]
%The Poincaré ball model is presented in Appendix \ref{app:ball_model}.

\paragraph{Region-based Methods}
Region-based methods embed the nodes of a directed acyclic graph (DAG) into regions on Euclidean space or Riemannian manifolds, such as boxes \cite{tBox,DBLP:conf/nips/ZhangBMM22},
balls \cite{DBLP:conf/icml/0008TO19}, and cones \cite{OE}, capturing hierarchical relationships through set-inclusion between these regions. Training is typically conducted using an inclusion loss, which is often defined in terms of the distance or volume of the regions \cite{OE}. 
However, such loss functions may involve non-smooth operations, like maximization, which have been addressed through probabilistic \cite{DBLP:conf/acl/McCallumVLM18,DBLP:conf/nips/DasguptaBZV0M20} or smooth \cite{tBox} approximations to improve performance. 
However, most of these methods suffer from the crowding effect in Euclidean space, which limits their expressiveness compared to our method and hyperbolic-based approaches.

% Due to the crowding effect in the Euclidean space, region-based methods could exhibit suboptimal performance in low-dimensional spaces or when representing large DAGs. To address this challenge, \cite{DBLP:conf/icml/0008TO19} proposed using balls on Riemannian manifolds, enabling more flexible shapes beyond the canonical balls used in Euclidean space.  
% However, their approach is more restrictive than ours, as it is limited to specific geometries and relies on a predefined Riemannian manifold, whereas our method can be applied to arbitrary regions without such dependency. 

% While \cite{tBox} demonstrates that box-based embeddings can represent any DAG, their focus is on reconstructing DAGs rather than preserving the graph's transitive properties. Consequently, this approach does not guarantee the faithful capture of hierarchical semantics unless all transitive closure edges are included in the training data.

 % where $t = \left( \sqrt{ \sum_{i=1}^{n-1} (u_i - v_i)^2 } - u_n \sinh{\sqrt{kr}} \right) / v_n.$ and 
 %    $H(\mathbf{v}, \mathbf{u}) = v_n^2 (1 + t^2) - u_n^2 \cosh^2{\sqrt{kr}}.$ 

\subsubsection{Hyperbolic Methods}  
% Hyperbolic space representations \cite{DBLP:conf/infocom/Kleinberg07,DBLP:journals/corr/abs-1009-0267,DBLP:conf/infocom/PapadopoulosKBV10} have been studied long ago. 
Hyperbolic space embeddings \cite{DBLP:conf/nips/NickelK17,DBLP:conf/nips/SonthaliaG20,DBLP:conf/icml/SalaSGR18} have been introduced to model hierarchical structures by leveraging favorable properties of hyperbolic space, such as its exponential growth. However, these methods primarily rely on pointwise distances to encode hierarchies, which is less intuitive for representing transitivity properties. In contrast, \textit{EntailmentCones} \cite{EntCone}, and \textit{ShadowCones} propose using regions in hyperbolic space that capture hierarchies through set inclusion, which can be regarded as region-based methods in hyperbolic spaces. 
Specifically, 
\textit{EntailmentCones} introduces closed-form hyperbolic cones, defined by their apex coordinates, while \textit{ShadowCones} use different cones are inspired by the physical interplay of light and shadow. However, as shown in Table~\ref{tab:region_hyperbolic}, these regions all take specific forms, limiting their generalizability to other tasks or methods.

\begin{figure*}[t]
    \centering
    % Subfigure 1: Taxonomy tree
    \begin{subfigure}[b]{0.3\textwidth}
        \centering
        \begin{tikzpicture}[scale=0.8]
            % Root node
            \node (R1) at (0, 1.5) {\textit{lives}};
            % Level 1 nodes
            \node (L1) at (-1, 0) {\textit{animal}};
            \node (L3) at (1, 0) {\textit{plant}};
            % Level 2 nodes
            \node (L4) at (-2.2, -1.5) {$\cdots$};
            \node (L) at (-1, -1.5) {\textit{dog}};
            \node (L5) at (0.2, -1.5) {\textit{human}};
            % Edges
            \draw (R1) -- (L1);
            \draw (R1) -- (L3);
            \draw (L1) -- (L4);
            \draw (L1) -- (L5);
            \draw (L1) -- (L);
        \end{tikzpicture}
        %\caption{Taxonomy tree}
    \end{subfigure}
    \hfill
    % Subfigure 2: Embedding in hyperbolic space
    \begin{subfigure}[b]{0.3\textwidth}
        \centering
        \begin{tikzpicture}[scale=0.8]
            % Gradient bar
            \shade[top color=white, bottom color=black] (-3, -2) rectangle (-2.7, 1.5);
            % Nodes
            \node (R1) at (0, 1) {$v_{\textit{lives}}$};
            \node (L1) at (-0.5, 0.2) {$v_{\textit{animal}}$};
            \node (L3) at (0.9, 0.1) {$v_{\textit{plant}}$};
            \node (L4) at (-0.9, -1.5) {$v_{\textit{dog}}$};
            \node (L) at (-1.8, -1.5) {$\cdots$};
            \node (L5) at (0, -1.5) {$v_{\textit{human}}$};
            % Edges
            % \draw (R1) -- (L1);
            % \draw (R1) -- (L3);
            % \draw (L1) -- (L4);
            % \draw (L1) -- (L5);
            % \draw (L1) -- (L);
         \draw[dashed] (R1) to[out=-140, in=85] (L1);
         \draw[dashed] (R1) to[out=-20, in=100] (L3);
         \draw[dashed] (L1) to[out=-115, in=85] (L4);
         \draw[dashed] (L1) to[out=-60, in=95] (L5);
         \draw[dashed] (L1) to[out=-145, in=80] (L);

            % Labels
            \node at (-0.5, 1.6) {$\mathbb{H}^2$ {\small (half-space model)}};
            \draw[thick, ->] (-2.5, -2) -- (2.3, -2);
            %\node[below] at (0, -2) {Boundary ($x_2 = 0$)};
        \end{tikzpicture}
        %\caption{Embedding in hyperbolic space}
    \end{subfigure}
    \hfill
    % Subfigure 3: Embedding in Euclidean space (boxes)
    \begin{subfigure}[b]{0.33\textwidth}
        \centering
        \begin{tikzpicture}[scale=0.8]
            % Outer box
            \draw (0.5,0.3) rectangle (5.2,-2.7);
            % Animal box
            \draw (0.7,0) rectangle (3.5,-2.5);
            % Inner boxes
            \draw[fill=gray!5, draw=black] (1.3,-2.3) rectangle (2.1,-0.8);
            \draw[fill=gray!5, draw=black] (2.2,-2.3) rectangle (3.4,-0.8);
            % plant box
            \draw (3.6,-2.5) rectangle (5,-0.7);
            % Box labels
                %\node at (2.5,0.6) {$B_{\textit{lives}}$};
            \node at (1.7,-1.5) {$B_{\textit{dog}}$};
            \node at (1,-1.5) {$\cdots$};
            \node at (1.8,-0.4) {$B_{\textit{animal}}$};
            \node at (2.8,-1.5) {$B_{\textit{human}}$};
            \node at (4.2,-0.1) {$B_{\textit{lives}}$};
            \node at (4.2,-1) {$B_{\textit{plant}}$};
            % Space label
            %\node at (6.2, 0.2) {$\mathbb{R}^2$};

            % Draw the x-axis
\draw[thick,->] (-0.2,-3) -- (5,-3) node[anchor=west] {$\mathbb{R}^2$};
% Draw the y-axis
\draw[thick,->] (0,-3.2) -- (0,0) node[anchor=south] {};

        \end{tikzpicture}
        %\caption{Embedding in Euclidean space (boxes)}
    \end{subfigure}
    \caption{Demonstration of a taxonomy (left), its embeddings in the hyperbolic space (middle) and in the Euclidean space as boxes (right).}
    \label{fig:depthD_motivation}
\end{figure*}

\subsubsection{Ontologies}\label{app:ont} 
Ontologies use sets of statements (axioms) about concepts (unary predicates) and roles (binary predicates) for knowledge representation and reasoning. 
We focus on commonly used \(\mathcal{EL}\)-ontologies, defined as follows. Let \(\NC = \{A, B, \dots\}\), \(\NR = \{r, t, \dots\}\), and \(\NI = \{a, b, \dots\}\) be pairwise disjoint sets of \emph{concept names} (also called \emph{atomic concepts}), \emph{role names}, and \emph{individual names}, respectively. \emph{\(\mathcal{EL}\)-concepts} are recursively defined from atomic concepts, roles, and individuals as follows:
\(
\top \mid \bot \mid A \mid C \sqcap D \mid \exists r. C \mid \{a\}
\)
An \emph{\(\mathcal{EL}\)-ontology} is a finite set of TBox axioms of the form 
\(C \sqsubseteq D.
\)
Note that here \(\sqsubseteq\) denotes ``subsumption'', which is a transitive relation that can be considered a partial order \(\prec\) by regarding \(C \sqsubseteq D\) as 
\(D \prec C\).

% \begin{example}
% From atomic concepts \(\text{Father}\), \(\text{Child}\), \(\text{Male}\), \dots and the role \(\text{hasParent}\), we can construct a small family ontology consisting of two TBox axioms:
% \[
% \text{Father} \sqsubseteq \text{Male} \sqcap \text{Parent}, \quad \text{Child} \sqsubseteq \exists \text{hasParent.Father}.
% \]
% \end{example}
% One can obtain additional axioms from a given ontology $\Omc$ by logical inference.
% For example, if $\Omc$ contains $A \sqsubseteq B$ and $B \sqsubseteq C$, then
% $A \sqsubseteq C$ can be derived; this is denoted by $\Omc \models A \sqsubseteq C$.
% A rigorous definition of the entailment relation ``$\models$'', grounded in the
% notions of \emph{interpretations} and \emph{models}, can be found
% in~\cite{DBLP:books/daglib/0041477}.

\subsubsection{Ontology Embeddings}
Ontology embeddings aim to encode ontology entities into numerical vector representations to support a variety of downstream tasks, such as axiom prediction and logical inference. Existing approaches can be broadly categorized into two types: \textit{geometric model-based} methods \cite{yang2025transbox,DBLP:conf/www/JackermeierC024,peng_description_2022,kulmanov_embeddings_2019}, which represent ontology entities as geometric objects and construct a geometric model capturing the semantic structure of the target ontology ; and \textit{language model-based} methods \cite{yang2025language,DBLP:journals/ml/ChenHJHAH21}, which learn embeddings from textual information, typically by leveraging pretrained large language models (LLMs).

% As discussed in the main text, the ontology embedding methods ELBE and ELEM can be adapted to serve as baseline methods for embedding DAGs using boxes and balls, respectively. Specifically, for a pair of nodes \((u, v)\), the energy function used during training is defined as follows:

% \begin{enumerate}
%     \item \textbf{ELBE:} Nodes \(u\) and \(v\) are embedded as boxes \(\boxes_u\) and \(\boxes_v\), respectively. The energy function is given by:
%     \[
%     E(u, v) = \left\| \max\{ |\bc_u - \bc_v| + \bo_v - \bo_u, \mathbf{0} \} \right\|,
%     \]
%     where \(\bc_u\) and \(\bc_v\) denote the center vectors of the boxes, and \(\bo_u\) and \(\bo_v\) denote the offsets of the boxes.

%     \item \textbf{ELEM:} Nodes \(u\) and \(v\) are embedded as balls \(\Ball_u\) and \(\Ball_v\), respectively. The energy function is defined as:
%     \[
%     E(u, v) = \max\{ \|\bc_u - \bc_v\| + r_v - r_u, \mathbf{0} \},
%     \]
%     where \(\bc_u\) and \(\bc_v\) represent the centers of the balls, and \(r_u\) and \(r_v\) represent their radii.

%     Note that the original ELEM method includes a regularization term that enforces the norms \(\|\bc_u\|\) and \(\|\bc_v\|\) to be close to 1. However, for the DAG case, we omit this regularization primarily applies to scenarios involving relation embeddings, such as axioms of the form \(A \sqsubseteq \exists r.B\) or \(\exists r.B \sqsubseteq A\), as in the KGE methods TransE \cite{DBLP:conf/nips/BordesUGWY13}.
% \end{enumerate}

\section{Method}

\subsection{Depth Dissimilarity: A Similarity Measure Incorporating Depth}\label{sec:dep_dist}
Unlike the Euclidean space, which is constrained by crowding effects that limit its embedding capacity, the hyperbolic space leverages exponential growth in distance and volume to offer superior embedding capabilities. This property makes the hyperbolic space particularly effective for representing tree-structured data in low-dimensional spaces. Notably, two key distinctions arise between region-based embeddings in the Euclidean space and hyperbolic embeddings:

\begin{enumerate}[leftmargin=*]
\item \textit{The hyperbolic space better discriminates different hierarchical layers than the Euclidean space.}

Consider the taxonomy illustrated in Figure \ref{fig:depthD_motivation} (left). Intuitively, since \textit{human} is a subcategory of \textit{animal}, the semantic difference between \textit{human} and \textit{plant} should be greater than that between \textit{animal} and \textit{plant}. The hyperbolic space effectively captures this hierarchical relationship by permitting \(d(v_{\textit{human}}, v_{\textit{plant}}) > d(v_{\textit{animal}}, v_{\textit{plant}}) + \Delta\), where \(\Delta\) is an arbitrary gap. This arises from the property that the distance metric diverges to infinity near the boundary, as illustrated by the shadow dense of the bar on the left-hand side of \(\mathbb{H}^2\) in Figure \ref{fig:depthD_motivation} (middle). In contrast, region-based embeddings on Euclidean space may violate this hierarchical constraint, potentially placing the box \(B_{\textit{human}}\) close to \(B_{\textit{plant}}\), resulting in a distance (\textit{e.g.}, the Euclidean distance between box centers) 
similar to or even smaller than that between \(B_{\textit{animal}}\) and \(B_{\textit{plant}}\).

\item \textit{The hyperbolic space enables the distinct representation of an arbitrary number of child nodes.}

As demonstrated in Figure \ref{fig:depthD_motivation}, the box embeddings in the Euclidean space face inherent limitations when representing multiple children of a node, such as \textit{animal}. As the number of children increases, their corresponding boxes must cluster within \(B_{\textit{animal}}\), leading to crowding. In contrast, the hyperbolic space can accommodate an arbitrary number of children while maintaining distinct separations between them. This capability arises from the exponential growth of distance near the boundary of the hyperbolic space, which allows unlimited child nodes to be positioned distinctly by placing them progressively closer to the boundary while preserving meaningful distances between them.
\end{enumerate}

% However, hyperbolic embeddings face two primary challenges: \textbf{(1)} The effectiveness of hyperbolic embeddings is contingent on specific geometric constructions, limiting their applicability to certain types of Directed Acyclic Graphs (DAGs). For instance, our evaluation on real-world datasets demonstrates that when graph structures deviate significantly from tree-like structure, hyperbolic embeddings may exhibit suboptimal performance; and \textbf{(2)} The computational complexity of hyperbolic embeddings surpasses that of Euclidean embeddings, requiring specialized non-linear functions and high-precision tensor operations, often performed with double-precision tensors for enhanced numerical accuracy.

In the following sections, we introduce the notion of \textit{depth dissimilarity} for regions in Euclidean space, which explicitly accounts for their ``size.'' We show that this measure retains advantages analogous to those of hyperbolic spaces (Theorem~\ref{theo:dep_dist}), while offering a simpler structure that facilitates implementation (e.g., avoiding specialized techniques such as double-precision tensors and Riemannian-specific optimizers \cite{DBLP:conf/iclr/BecigneulG19}) and enabling potential extensions e.g., integration with ontology embeddings; see Section~\ref{sec:Ont_inf}). Moreover, we establish that depth dissimilarity encompasses hyperbolic distance as a special case (Proposition~\ref{prop:dep_eq}), while providing greater flexibility by preserving key properties of hyperbolic geometry through simple polynomial functions (Proposition \ref{prop:hyper}).

\subsubsection{Construction}
The depth dissimilarity serves as a similarity measure that quantifies the relationship between objects considering their hierarchical depth. As we use set-inclusion relations to model the hierarchy, this hierarchical depth can be represented through the size of the regions, such as their volumes or diameters. 
Formally, the depth dissimilarity is defined as follows, where the size of the regions is represented by a function $f(\reg)$:

\begin{defi}[Depth Dissimilarity]\label{def:dep}
Let $\Regions$ be a collection of \emph{parameterized regions}\footnote{The formal definition is defer to Section \ref{app:general_reg} for simplicity. Examples as boxes and balls are shown in Example \ref{exp:dist_depth_ball}, \ref{exp:dist_depth_box}.} in the $n$-dimensional Euclidean space $\R^n$, i.e., each region $\reg \in \Regions$ is characterized by an $m$-dim parameter $\para(\reg) \in \mathbb{R}^m$. The depth dissimilarity between two regions $\reg_1,\reg_2 \in \Regions$, is defined as:
\begin{equation}\label{eq:depth_dist}
\ddepth(\reg_1, \reg_2) = g\left(\frac{\|\para(\reg_1) - \para(\reg_2)\|_p^p}{f(\reg_1) f(\reg_2)}\right)
\end{equation}
where $\|\cdot\|_p$ is the $p$-norm (\textit{i.e.}, $||\x||_p=(\sum_i\x_i^p)^{1/p}$), and:
\begin{itemize}[leftmargin=*]
%\item $\|\cdot\|_p$ denotes the $p$-norm for some $p >0$,
\item $g: \mathbb{R}_{\geq 0} \to \mathbb{R}_{\geq 0}$ is an increasing function such that $g(x) = 0$ if and only if $x = 0$,
\item $f: \mathcal{R} \to \mathbb{R}_{>0}$ is a function that measures the size of regions. It satisfies:
\(
\lim_{\reg \to \emptyset} f(\reg) = 0\footnote{$\reg \to \emptyset$ indicates that the region $\reg$ shrinks to the empty set (\textit{e.g.}, its volume or diameter tends to zero).}.
\)
\end{itemize}
\end{defi}

We require \(f\) and \(g\) to have non-negative values to ensure the depth dissimilarity is non-negative. Additionally, we stipulate that \(\lim_{\reg \to \emptyset} f(\reg) = 0\) to guarantee that as a region shrinks to an empty set, the dissimilarity between this object and others can approach infinity. This setting emulates the beneficial properties of the hyperbolic space, where the dissimilarity between two points can grow rapidly as they approach the boundary of the space (i.e., $\x_n=0$ in the Poincaré half-space model). In our context, the boundary of the space $\mathcal{R}$ of a collection of (parametrized) regions in the Euclidean space is the empty set. By selecting an appropriate function \(f\), we can control the rate at which the dissimilarity between two objects increases as they approach this boundary.

For clarity and tractability, we focus on the most representative region types, namely boxes and balls, to provide illustrative examples and empirical implementations. We defer the theoretical analysis of general parametrized region representations to Section~\ref{app:general_reg}.

\begin{example}\label{exp:dist_depth_ball}
Let 
\(
\Ball(\bc, r) = \{\x \mid \|\x - \bc\| \leq r\}
\) be a \emph{ball} defined by a center $\bc \in \R^n$ and a radius $r > 0$. 
By setting $g(x) = x$ and $f(\Ball) = r$, we obtain a depth dissimilarity of the form:
\begin{equation}\label{eq:dep_ball}
\ddepth(\Ball_1(\bc_1, r_1), \Ball_2(\bc_2, r_2)) = \frac{||\bc_1 - \bc_2||_p^p + |r_1 - r_2|^p}{r_1 \cdot r_2}.
\end{equation}
\end{example}

\begin{example}\label{exp:dist_depth_box}
Let $\boxes(\bc, \bo) = \{\x \in \R^n \mid \bc - \bo \leq \x \leq \bc + \bo\}$ be a \emph{box} defined by a center $\bc \in \R^n$ and an offset $\bo \in \R^n_{> 0}$. 
By setting $g(x) = x$ and $f(\boxes) = \|\bo\|$, we obtain a depth dissimilarity:
\begin{equation}\label{eq:dep_box}
\ddepth(\boxes_1(\bc_1, \bo_1), \boxes_2(\bc_2, \bo_2)) = \frac{||\bc_1 - \bc_2||_p^p + ||\bo_1 - \bo_2||_p^p}{\|\bo_1\| \cdot \|\bo_2\|}.
\end{equation}
\end{example}

The following result highlights that the depth dissimilarity exhibits properties analogous to those of the hyperbolic distance discussed earlier in this section. Specifically, the depth dissimilarity \textbf{(1)} effectively distinguishes concepts across different layers of the hierarchy with an arbitrary separation gap, denoted by $\Delta$ in item 1 of the following Theorem \ref{theo:dep_dist}, and \textbf{(2)} distinctly represents an arbitrary number $n$ of children within a single parent by a dissimilarity greater than $M$, as demonstrated in item 2 of the theorem. 

% \begin{theo}\label{theo:dep_dist}
%     Consider the region space \(\Regions\) consisting of \textbf{balls}, with the depth distance defined in Example \ref{exp:dist_depth_ball}. Then, the following hold:
%     \begin{enumerate}
%         \item For any \(\reg_1, \reg_2 \in \Regions\) and $N>0$, there exists a positive constant \(r_0\) such that for any ball \(\Ball(\bc', r') \subseteq \reg_2\), if \(r' \leq r_0\), we have 
%         \[\ddepth(\reg_1, \Ball(\bc', r')) > \ddepth(\reg_1, \reg_2) + N.\]
%         \item For any \(\reg \in \Regions\) and any integers \(n, M > 0\), there exist subsets \(\reg_1, \ldots, \reg_n \subseteq \reg\) such that \(\ddepth(\reg_i, \reg_j) > M\) for any distinct \(i, j \in \{1, \ldots, n\}\).
%     \end{enumerate}
%     The same conclusions hold for boxes, where \(r_0 \in \mathbb{R}_{> 0}\) is replaced with \(\bo_0 \in \mathbb{R}_{> 0}\), and the condition \(r' \leq r_0\) is replaced by \(||\bo'|| \leq r_0\).
% \end{theo}

\begin{figure*}[t]
    \centering
    \begin{minipage}{0.48\textwidth}
        \centering
        \begin{tikzpicture}[scale=0.6]
            % Outer ellipse
            \begin{scope}[rotate around={35:(0,0)}]
            \fill[color=gray!20,  opacity=0.5]
            (0,0) ellipse (2 and 2);
            \draw[color=black, line width=1pt]
            (0,0) ellipse (2 and 2);
            
            % Inner ellipse
            \fill[color=gray!100, opacity=0.5]
            (0.3,0.0) ellipse (1 and 0.8);
            \draw[color=black, line width=1pt]
            (0.3,0.0) ellipse (1 and 0.8);
            \coordinate  (A) at (1.3, 0.0);
            \coordinate (B) at (2, 0);
            \draw[line width=2pt, red] (A) -- (B);
    
            \coordinate (A1) at (-0.7, 0);
            \coordinate (B1) at (-2, 0);
            \draw[dashed, line width=1pt, green] (A1) -- (B1);
            \end{scope}
    
            % add notation
            \node at (1.6, 1.9) {$\reg_1^c$};
            \node at (-1, 1) {$\reg_1$};
            \node at (0.3, 0.1) {$\reg_2$};
    
            % Outer ellipse
            \begin{scope}[shift={(4.5,0,0)}, rotate around={35:(0,0)},]
            \fill[color=gray!20,  opacity=0.5]
            (0,0) ellipse (2 and 2);
            \draw[color=black, line width=1pt]
            (0,0) ellipse (2 and 2);
    
            % Inner ellipse
            \fill[color=gray!100,  opacity=0.5]
            (1.8,0.0) ellipse (1 and 0.8);
            \draw[color=black, line width=1pt]
            (1.8,0.0) ellipse (1 and 0.8);
            \coordinate  (A) at (2.8, 0.0);
    
            \coordinate (B) at (2, 0);
            \draw[line width=2pt, red] (A) -- (B);
    
            \coordinate  (A1) at (0.8, 0.0);
            \coordinate (B1) at (-2, 0);
            \draw[dashed, line width=1pt, green] (A1) -- (B1);
            \end{scope}
    
            % add notation
            \node at (3.5, 0.5) {$\reg_1$};
            \node at (6.8, 0.3) {$\reg_2$};
        \end{tikzpicture}
        \caption{Illustration of $\dboundary(\reg_1,\reg_2)$ (red) for $\reg_2 \subseteq \reg_1$ (left) or $\reg_2 \not\subseteq \reg_1$ (right). Green lines shows the inverse: $\dboundary(\reg_2,\reg_1)$.}
        \label{fig:boundary_dist}
    \end{minipage}
    \hfill
    \begin{minipage}{0.48\textwidth}
        \centering
        \begin{tikzpicture}[scale=0.6]
            \fill[color=gray!20,  opacity=0.5]
                (0,0) ellipse (2.5 and 2);
            \draw[color=black, line width=1pt]
                (0,0) ellipse (2.5 and 2);
            
            \begin{scope}[rotate around={35:(0,0)}]
                % Inner ellipse
                \fill[color=gray!60, opacity=0.5]
                (0.5,0.0) ellipse (1.7 and 1.2);
                \draw[color=black, line width=1pt]
                (0.5,0.0) ellipse (1.7 and 1.2);
            \end{scope}
    
            \node at (0.2, -0.3) {$\reg_2$};
            \node at (-1.8, 0.5) {$\reg_1$};
            
            \begin{scope}[shift={(5.4,0,0)}]
                \fill[color=gray!20,  opacity=0.5]
                (0,0) ellipse (2.5 and 2);
                \draw[color=black, line width=1pt]
                (0,0) ellipse (2.5 and 2);
        
                % Outer ellipse
                \begin{scope}[rotate around={35:(0,0)}]
                    % Inner ellipse
                    \fill[color=gray!60, opacity=0.5]
                    (0.3,0.0) ellipse (1.7 and 1.2);
                    \draw[color=black, line width=1pt]
                    (0.3,0.0) ellipse (1.7 and 1.2);
                   
                    % Inner ellipse
                    \fill[color=gray!100, opacity=0.5]
                    (0.5,0.0) ellipse (0.8 and 0.8);
                    \draw[color=black, line width=1pt]
                    (0.5, 0.0) ellipse (0.8 and 0.8);
                \end{scope}
        
                % add notation
                \node at (-1.8, 0.5) {$\reg_1$};
                \node at (-0.5, -0.7) {$\reg_2$};
                \node at (0.3, 0.1) {$\reg_3$};
            \end{scope}
        \end{tikzpicture}
        \caption{Illustration of internally tangent of item 1 (left) and item 2 (right) in Proposition \ref{prop:depth_hyper}.}
        \label{fig:boundary_dist_prop}
    \end{minipage}
\end{figure*}

\begin{theo}\label{theo:dep_dist}
    Consider the region space \(\mathcal{B}^n\) consisting of \textbf{balls} in $\R^n$, with the depth dissimilarity defined in Example \ref{exp:dist_depth_ball}. The following properties hold:
    \begin{enumerate}[leftmargin=*]
        \item For any \(\Ball_1, \Ball_2 \in \mathcal{B}^n\) and any \(\Delta > 0\), there exists a positive constant \(r_0\) such that for any \(\Ball(\bc', r') \subseteq \Ball_2\), if \(r' \leq r_0\), then 
        \(
        \ddepth(\Ball_1, \Ball(\bc', r')) > \ddepth(\Ball_1, \Ball_2) + \Delta.
        \)
        \item For any \(\Ball \in \mathcal{B}^n\), any integer \(n\) and any \(M > 0\), there exist subsets \(\Ball_1, \ldots, \Ball_n \subseteq \Ball\) such that for any distinct \(i, j \in \{1, \ldots, n\}\), we have
        \(
        \ddepth(\Ball_i, \Ball_j) > M.
        \)
        
    \end{enumerate}
    The same conclusions hold for boxes, where \(r_0 \in \mathbb{R}_{> 0}\) is replaced with \(\bo_0 \in (\R_{>0})^n\), and the condition \(r' \leq r_0\) is replaced by \(\bo' \leq \bo_0\) element-wise.
\end{theo}

% We defer the proof to Section~\ref{app:general_reg}, where a more general result is established that applies to the box and ball embeddings introduced above.
\begin{proof}
\textbf{For item 1:} Let \(\reg_1 = \Ball(\bc_1, r_1)\) and \(\reg_2 = \Ball(\bc_2, r_2)\). For any ball \(\Ball(\bc', r') \subseteq \Ball(\bc_2, r_2)\) with \(r' \leq 0.5 \cdot r_1\), we have:
     \(
     ||\bc_1 - \bc'||^p_p + |r_1 - r'|^p > (0.5 \cdot r_1)^p.
     \)
    Therefore:
    \begin{align*}
     \ddepth(\reg_1, \Ball(\bc', r')) = \frac{||\bc_1 - \bc'||^p_p + |r_1 - r'|^p}{f(\reg_1) f(\Ball(\bc', r'))}\\
     > \frac{\left( 0.5 \cdot  r_1 \right)^p}{f(\reg_1) f(\Ball(\bc', r'))}.
    \end{align*}
    Thus, when $f(\Ball(\bc', r')) < \epsilon$ with 
    \[
    \epsilon := \frac{\left( 0.5 \cdot  r_1 \right)^p}{f(\reg_1) [\ddepth(\reg_1, \reg_2)+\Delta]},
    \]
    we have \(\ddepth(\reg_1, \Ball(\bc', r')) > \ddepth(\reg_1, \reg_2) + \Delta\). 
    Since 
    \(
    \lim_{\reg \to \emptyset} f(\reg) = 0
    \) by assumption, 
     there exists a \(\delta > 0\) such that when \(r < \delta\), we have \(f(\reg) \leq \epsilon\) for any region \(\reg = \Ball(\bc, r)\). In conclusion, \(r_0 = \min\{\delta, 0.5 \cdot r_1\}\) satisfies the required condition.

\textbf{For item 2: } Let \(\reg = \Ball(\bc, r)\). For any \(n, M > 0\), we can select \(n\) distinct vectors \(\bc_1, \ldots, \bc_n \in \Ball(\bc, 0.5 \cdot r)\) such that: \(\| \bc_i - \bc_j \|_p^p > \delta > 0 \) for some \(\delta > 0\).

    Similarly to item 1, we can choose \(r_i\) small enough such that for any ball \(\Ball(\bc_i, r_i)\), we have:
    \(
    f(\Ball(\bc_i, r_i)) < \left( \delta/M \right)^{0.5}.
    \)
    Thus, $ \ddepth(\Ball(\bc_i, r_i), \Ball(\bc_j, r_j))$ is bigger than: 
    \begin{align*}
      \frac{\delta}{f(\Ball(\bc_i, r_i)) f(\Ball(\bc_j, r_j))}\geq M.
    \end{align*}

The proof for the case of boxes follows the same reasoning, with the radius \(r\) replaced by the offset \(\bo\) or its norm $||\bo||$.
\end{proof}

% It is worth noting that while we focus on boxes and balls for simplicity in parameterization, analogous results in Theorem \ref{theo:dep_dist} are expected to hold for other parameterized regions.

% This theorem demonstrates that the depth distance shares key advantages with the hyperbolic distance. Specifically:
% \begin{itemize}
%     \item \textbf{Item 1:} The depth distance allows for distinguishing between nested regions with arbitrarily large separation gaps. This property ensures that smaller subregions nested within a larger parent can be separated by an adjustable margin \(N\).
%     \item \textbf{Item 2:} The depth distance effectively represents an arbitrary number \(n\) of child regions within a single parent region. It ensures that their pairwise distances exceed a desired threshold \(M\), facilitating the distinct representation of multiple children.
% \end{itemize}

\subsubsection{Comparison with the Hyperbolic Distance}
The following theorem shows that our depth dissimilarity encompasses hyperbolic distance as a special case, obtained with specific choices of functions \( f, g \) in Definition \ref{def:dep}.

\begin{prop}\label{prop:dep_eq}  
Let \(\mathbb{H}^{n+1}\) be the hyperbolic space with curvature \(-1\). Assume the ball space \(\mathcal{B}^n\) is parameterized as in Example \ref{exp:dist_depth_ball}, and equipped with the depth dissimilarity defined in Equation (\ref{eq:depth_dist}).  Then the map $F \colon \mathcal{B}^n \to \mathbb{H}^{n+1} $ defined by $F(\Ball(\bc, r))= [\bc : r]$ 
% \begin{align*}
%     F \colon \quad\quad \mathcal{B}^n &\to \mathbb{H}^{n+1} \\
%     \Ball(\bc, r) &\mapsto [\bc : r]
% \end{align*}
is an \emph{bijective isometry} when \( p = 2 \), \( g(x) = \operatorname{arcosh}(x + 1) \), and \( f(\Ball(\bc, r)) = \sqrt{2}\,r \). % in Equation (\ref{eq:depth_dist}).
\end{prop}

\begin{proof}
Recall that when the curvature is \(-1\), the distance in the hyperbolic space (half-plane model) is given by:
\[
d_k(\x, \y) = \operatorname{arcosh}\left( 1 + \frac{\|\x - \y\|^2_2}{2 \x_n \y_n} \right).
\]
Then, the distance induced by the function \(F\) is of the form:
\begin{align}
  d^\#_{\mathcal{B}^n}(\Ball(\bc, r), &\Ball'(\bc'. r')) = \\
  &\operatorname{arcosh}\left( 1 + \frac{\|\bc - \bc'\|^2_2 + (r - r')^2}{2 r r'} \right).  
\end{align}

This coincides with the depth dissimilarity in Example \ref{exp:dist_depth_ball} when \( p = 2 \), \( g(x) = \operatorname{arcosh}(x + 1) \), and \( f(\Ball(\bc, r)) = \sqrt{2}\,r \). This completes the proof.
\end{proof}

Moreover, the following result demonstrates that even with a simple linear function $g(\cdot)$, our depth dissimilarity retains the ability to capture the hyperbolic structure. Specifically, there exists a map from the region space to the hyperbolic space that preserves the order of hyperbolic distances for all pairs of points, as stated below.

\begin{prop}\label{prop:hyper}Following Proposition \ref{prop:dep_eq},  let the depth dissimilarity $\ddepth$ be redefined by replacing \( g \) to \( g(x) = k\cdot x\, (k>0)\). Then the map \( F\) retains the following monotonicity property: for any points \(\x_1, \x_2, \x_3, \x_4 \in \mathbb{H}^n\),
\(
d_{\mathbb{H}^n}(\x_1, \x_2) < d_{\mathbb{H}^n}(\x_3, \x_4)
\)
if and only if
\[
\ddepth(F^{-1}(\x_1), F^{-1}(\x_2)) < \ddepth(F^{-1}(\x_3), F^{-1}(\x_4)).
\]
\end{prop}

\begin{proof}
This proposition follows directly from Proposition \ref{prop:dep_eq}. The function \( h(x) = \operatorname{arcosh}(x + 1) \) is an increasing bijection from \(\mathbb{R}_{\geq 0}\) to \(\mathbb{R}_{\geq 0}\). Thus, for any \( x, x' \geq 0 \), we have:
\[
x \leq x' \iff h^{-1}(x) \leq h^{-1}(x').
\]
By assumption, \(\ddepth(\cdot, \cdot) = k\cdot h^{-1}(d_{\mathbb{H}^n}(\cdot, \cdot))\) for some $k>0$. Therefore, for any points \(\x_1, \x_2, \x_3, \x_4 \in \mathbb{H}^n\), we have:
\(
d_{\mathbb{H}^n}(\x_1, \x_2) < d_{\mathbb{H}^n}(\x_3, \x_4)
\)
if and only if
\[
\ddepth(F^{-1}(\x_1), F^{-1}(\x_2)) < \ddepth(F^{-1}(\x_3), F^{-1}(\x_4)).
\]
\end{proof}

It is important to recognize that evaluation metrics such as F1 scores and Hits@k rely solely on the ranking induced by the scoring function. As a result, only the relative order of the scores matters, while their absolute values are less important. Therefore, preserving the same ranking as that produced by hyperbolic distances is sufficient to attain comparable performance. Consequently, the depth-based dissimilarity defined in Equations~\ref{eq:dep_ball} and~\ref{eq:dep_box}, with the function \( g(x) = x \), are well-suited for our implementation, as justified by Proposition~\ref{prop:depth_hyper}.

\subsection{Boundary Dissimilarity: A Non-Symmetric Measure of Inclusion} \label{sec:dboundary}

Although the depth dissimilarity $\ddepth$ introduced above has been shown to have great power for embedding hierarchical data, it is a symmetric metric and therefore inadequate for fully capturing the inherently non-symmetric hierarchical relationships between objects. To address this limitation, we introduce the boundary dissimilarity, specifically designed to reflect the partial order of regions defined by set inclusion. 

Our boundary dissimilarity generalizes the signed-distance-to-boundary in ShadowCone (Theorem 4.2, \cite{ShadowCone}) by extending its applicability from specialized hyperbolic cone geometries to arbitrary Euclidean regions. This generalization is formally defined in Definition~\ref{def:bd}. 
As a result, our boundary dissimilarity can be computed in a much simpler form (Example~\ref{exp:boundary_dist}) compared to the formulation used in ShadowCone (Equation~\ref{eq:cone_dist}).

\paragraph{Construction}
 The boundary dissimilarity is defined to measure the minimal cost associated with transforming the spatial relationship between two regions $\reg_1$ and $\reg_2$. Specifically, it quantifies the cost of moving $\reg_2$ out of $\reg_1$ when $\reg_2 \subseteq \reg_1$, or moving $\reg_2$ into $\reg_1$ otherwise (when $\reg_2 \not\subseteq \reg_1$). This cost can be defined for arbitrary geometric objects based on either distance or volume within Euclidean or other spaces. Below, we introduce a boundary dissimilarity based on the Euclidean distance for two regions $\reg_1, \reg_2 \subseteq \mathbb{R}^n$, which consists of two cases:

\begin{enumerate}[leftmargin=*]
\item \textit{Containment (\(\reg_2 \subseteq \reg_1\)):} As illustrated in the left of Figure~\ref{fig:boundary_dist}, when \(\reg_2\) is fully contained within \(\reg_1\), the boundary dissimilarity is defined by the minimum Euclidean distance between the complementary region \(\reg_1^c\) and the points in \(\reg_2\) (\textit{i.e.}, length of the red line). This distance quantifies the minimum translation cost required to move at least a part of \(\reg_2\) out of \(\reg_1\).

\item \textit{Non-Containment (\(\reg_2 \not\subseteq \reg_1\)):} As shown in the right of Figure~\ref{fig:boundary_dist}, when \(\reg_2\) is not fully contained within \(\reg_1\), the boundary dissimilarity is defined as the maximum Euclidean distance from the points in \(\reg_2 \setminus \reg_1\) to \(\reg_1\) (\textit{i.e.}, length of the red line). This distance quantifies the minimum translation cost for moving \(\reg_2\) into \(\reg_1\).
\end{enumerate}

Let $d(\reg, \x):=\min\{||\x-\y||_2\mid \y\in\reg\}$ be the distance of a points $\x$ to a region $\reg$ defined by the minimal distance from $\x$ to $\y\in\reg$.
The formal definition of the boundary dissimilarity is as follows:

\begin{defi}[Boundary Dissimilarity]\label{def:bd}Given a  region space \(\mathcal{R}\), we define the boundary dissimilarity over $\reg_1, \reg_2\in \mathcal{R}$ by:
% \begin{align}\label{eq:boundary_dist}
% \dboundary(\reg_1, &\reg_2) =\\\nonumber
% &\begin{cases}
% -\min\limits_{\x_2\in \reg_2}\{d(\reg_1^c, \x_2)\} & \text{ if }\reg_2\subseteq \reg_1\\
% \max\limits_{\x_2 \in \reg_2\setminus\reg_1}\{d(\reg_1, \x_2)\} & \text{ else }.
% \end{cases} 
% \end{align}
\begin{equation}\label{eq:boundary_dist}
\dboundary(\reg_1, \reg_2) = \begin{cases}
-\min\limits_{\x_2\in \reg_2}\{d(\reg_1^c, \x_2)\} & \text{ if }\reg_2\subseteq \reg_1\\
\max\limits_{\x_2 \in \reg_2\setminus\reg_1}\{d(\reg_1, \x_2)\} & \text{ else }.
\end{cases}
\end{equation}
\end{defi}
Note that a negative sign is added to \(\dboundary(\reg_1, \reg_2)\) in the containment case (\(\reg_2 \subseteq \reg_1\)) to clearly distinguish it from other cases. 
%Note that, by our definition, when \(\reg_2 \subseteq \reg_1\), \(\dboundary(\reg_1, \reg_2)\) takes a negative value, which can be easily distinguished from other cases. 
Moreover, the boundary dissimilarity is inherently asymmetric, that is, \(\dboundary(\reg_1, \reg_2) \neq \dboundary(\reg_2, \reg_1)\) in general. For example, as illustrated in Figure \ref{fig:boundary_dist}, the boundary dissimilarity in the reverse order, \(\dboundary(\reg_2, \reg_1)\), corresponds to the length of the green dashed line, which differs from the red line representing \(\dboundary(\reg_1, \reg_2)\).

Moreover, for widely used geometric objects such as balls and boxes, the boundary dissimilarity can be computed efficiently using simple arithmetic operations.

\begin{example}\label{exp:boundary_dist}
If \(\reg_1 = \Ball(\bc_1, r_1)\) and \(\reg_2 = \Ball(\bc_2, r_2)\), the boundary dissimilarity have the form (here the two cases can be unified into a single formula):
\[
\dboundary(\reg_1, \reg_2) = \|\bc_1 - \bc_2\|_2 + r_2 - r_1 \big.
\]
For the case of boxes, where \(\reg_1 = \boxes(\bc_1, \bo_1)\) and \(\reg_2 = \boxes(\bc_2, \bo_2)\), the boundary dissimilarity have the form:
\begin{align*}
\dboundary&(\reg_1, \reg_2) =\\
&\begin{cases}
   \max(|\bc_1 - \bc_2|+ \bo_2 - \bo_1) & \textit{ if } \reg_2\subseteq\reg_1,\\
  ||\max\{|\bc_1 - \bc_2| + \bo_2 - \bo_1, \mathbf{0}\}||_2 & \textit{else}.
\end{cases}
\end{align*}
Here, \(\max(\cdot)\) in the first line denotes the maximal value along all dimensions, while \(\max\{\cdot, \cdot\}\) in the second line applies element-wise to the two vectors or values.
\end{example}

% It might seem intuitive to add an additional term, such as \( \vol(\reg_2) - \vol(\reg_1)\), to account for volume differences. However, this is unnecessary because our boundary function already implicitly considers the diameter (as illustrated in Example \ref{exp:boundary_dist}). Consequently, it naturally encourages the volume \(\vol(\reg_2)\) to be smaller than \(\vol(\reg_1)\).

The following proposition demonstrates that our definition of the boundary dissimilarity effectively captures the inclusion relationship between two regions in two key aspects, as illustrated in Figure \ref{fig:boundary_dist_prop}: 
(i) it identifies whether one region is (exactly) contained within another, and (ii) it enhances discrimination in inclusion chains, as smaller regions tend to have larger boundary dissimilaritys. This property is useful for distinguishing shallow children from deeper ones. It is worth noting that the proposition below applies to any regions, not only boxes or balls.

\begin{prop}\label{prop:depth_hyper}
For the boundary dissimilarity \(\dboundary\) in Definition \ref{def:bd}, the following properties hold:
\begin{enumerate}[leftmargin=*]
    \item \(\dboundary(\reg_1, \reg_2) \leq 0\) if and only if \(\reg_2 \subseteq \reg_1\). Moreover, \(\dboundary(\reg_1, \reg_2) = 0\) if and only if \(\reg_1\) is \emph{internally tangent} to \(\reg_2\). That is, \(\reg_1 \subseteq \reg_2\), and their boundaries intersect at some point.
    \item \(
    \dboundary(\reg_1, \reg_3) \leq \dboundary(\reg_1, \reg_2)
    \) if \(\reg_3 \subseteq \reg_2 \subseteq \reg_1\).
    
\end{enumerate}
\end{prop}
\begin{proof}
We prove each item one-by-one:
\begin{enumerate}
    \item By definition, we have \(\dboundary(\reg_1, \reg_2) \leq 0\) if and only if \(\reg_2 \subseteq \reg_1\). Next, we focus on the case \(\dboundary(\reg_1, \reg_2) = 0\).

    Note that if \(\dboundary(\reg_1, \reg_2) = 0\), we must have $\reg_2\subseteq \reg_1$. Otherwise, we have 
    \[
    \dboundary(\reg_1, \reg_2) = \max\limits_{\x_2 \in \reg_2\setminus\reg_1}\{d(\reg_1, \x_2)\} = 0
    \]
    Therefore, for any $\x_2\in \reg_2$, we have $d(\reg_1, \x_2)=0$, therefore $\x_2\subseteq \reg_1$ (assuming $\reg_1$ is a closed set). Contradiction!

    Since $\reg_1\subseteq \reg_2$, we have 
    \[
    \dboundary(\reg_1, \reg_2) = \max\limits_{\x_2\in \reg_2}\{-d(\reg_1^c, \x_2)\} =0.
    \]
    Therefore, there must exist \(\x_2\in \reg_2\) such that \(d(\reg_1^c, \x_2)=0\), and thus \(\x_2\in \reg_1^c\). Since we have \(\x_2\subseteq \reg_2\subseteq \reg_1\), therefore, \(\x_2\in \partial(\reg_1)\). Similarly, since \(\x_2\subseteq \reg_1^c\subseteq \reg_2^c\) and \(\x_2\in \reg_2\), we also have \(\x_2\in \partial(\reg_2)\). This finishes the proof of the first case.

    \item By assumption, we have 
      \[
    \dboundary(\reg_1, \reg_2) = \max\limits_{\x_2\in \reg_2}\{-d(\reg_1^c, \x_2)\}, \]
    \[
    \dboundary(\reg_1, \reg_3) = \max\limits_{\x_2\in \reg_3}\{-d(\reg_1^c, \x_2)\}.
    \]
    Since $\reg_3\subseteq \reg_2$, of course we have \( \dboundary(\reg_1, \reg_3) \leq \dboundary(\reg_1, \reg_2). \)
    This finishes the proof of the second case.
\end{enumerate}
\end{proof}

\paragraph{Specific Constructions for Boxes or Balls} 
For specific geometric regions like balls and boxes, we can create specialized distance functions to measure the set-inclusion relationship based on their intrinsic geometric properties or established methods. Our framework accommodates these specialized metrics by allowing them to replace the general boundary dissimilarity function.

% For geometric regions such as balls and boxes, we could also construct specialized distance functions, which  derived either from the regions' intrinsic geometric properties or from established methods like ShadowCone~\cite{ShadowCone}, to characterize their set-inclusion relationships. These functions can be integrated into our framework by replacing our general boundary dissimilarity with these region-specific distances.

1. \textit{Volume-based dissimilarity for boxes:} Since the volume of a box can be computed as the product of its offsets along different dimensions, we can define a partial distance based on volume: 
%   \begin{align}\label{eq:vol_dist}
% d_{\text{vol}}(\reg_1, &\reg_2) =\\\nonumber
% &\begin{cases}
% \ln\left( \frac{\vol(\reg_1\setminus \reg_2 )}{\vol(\reg_1)} \right) & \text{if } \reg_2\subseteq \reg_1,\\
% -\ln\left( \frac{\vol(\reg_1 \cap \reg_2)}{\vol(\reg_2)} \right) & \text{else}.
% \end{cases}
%   \end{align}
    \begin{equation}\label{eq:vol_dist}
d_{\text{vol}}(\reg_1, \reg_2) = -\ln\left( \frac{\vol(\reg_1 \cap \reg_2)}{\vol(\reg_2)} \right).
\end{equation}
%The negative logarithm ensures that the distance is non-negative and equals zero only when \(\reg_2 \subseteq \reg_1\). 

% To improve performance and enhance training stability, \(\tau\)Box~\cite{tBox} proposed a smooth approximation of the volume-based distance by employing the LogSumExp function to approximate the maximum function used in calculating the intersection \(\reg_1 \cap \reg_2\).

% However, it is important to note that this volume-based construction has limitations. Specifically, it cannot effectively distinguish between cases where \(\reg_2 \subseteq \reg_1\), as \(d_{\text{vol}}(\reg_1, \reg_2) = 0\) for all such instances. Consequently, it fails to differentiate whether one region \(\reg_1\) is more strictly contained within \(\reg_2\) compared to another region \(\reg_1'\).
2.\textit{Hyperbolic dissimilarity for balls:} \cite{ShadowCone} introduced a series of circular cones in hyperbolic space and defined a boundary dissimilarity based on the hyperbolic distance between the apex of these cones. 
By utilizing the natural mapping from balls to circular cones, we can derive a new boundary dissimilarity for balls as follows:
\begin{align}\label{eq:cone_dist}
\dboundary^{\text{cone}}(\reg_1, \reg_2) = \arcsinh&\left(\frac{||\bc_1- \bc_2||_2-r_1}{r_2} \right)\\\nonumber
&+ \arcsinh(1)
\end{align}

\subsection{Training}\label{sec:train}
 For a given pair \((u, v)\), we define their energy as a weighted sum with  weight \(\lambda\) that balances the contributions of the hyperbolic-like depth dissimilarity:  
\begin{equation}\label{eq:energy}  
    E(u, v) = \dboundary(\reg_u, \reg_v) + \lambda \cdot \ddepth(\reg_u, \reg_v),  
\end{equation}  
We say that \( u \) is considered a parent of \( v \) (i.e., $u\prec v$) if \( E(u, v) \leq t \), where \( t \) is a threshold that achieved the best performance on the evaluation set.

% This formulation ensures the simultaneous preservation of hyperbolic structural properties when needed, while maintaining the asymmetry characteristic of hierarchical relationships. 
For model training, we use $\ddepth$ from Equation \ref{eq:dep_ball} or \ref{eq:dep_box} with the contrastive loss from \cite{ShadowCone}: 
\begin{align}\label{eq:loss}
\mathcal{L}(\gamma_1,&\gamma_2) = \sum_{(u, v) \in P}\Bigg( \max \{ E(u, v), \gamma_1\} \\\nonumber
 &+ \log\bigg(\sum_{(u, v') \in N}  e^{\max\{\gamma_2 - \dboundary(\reg_u, \reg_{v'}), 0\}} \bigg)\Bigg),
\end{align}
where \(P\) and \(N\) denote positive and negative sample pairs, respectively. 
%The margins \(\gamma_1\) and \(\gamma_2\) control how closely the embedding regions should be positioned for positive and negative pairs. 
For \textbf{positive pairs} \((u, v) \in P\), the loss based on \(E(u, v)\) promotes both the containment of \(\reg_v\) within \(\reg_u\) (via the \(\dboundary\) term) and their geometric similarity (via the \(\ddepth\) term), whose contributions are controlled by the weight \(\lambda\) and the threshold \(\gamma_1\). 
For \textbf{negative pairs} \((u, v') \in N\), since our primary goal is to push \(\reg_{v'}\) outside of \(\reg_u\), it is sufficient to use the boundary dissimilarity \(\dboundary(\reg_u, \reg_{v'})\) rather than the energy \(E(u, v')\).  A threshold \(\gamma_2\) is used to regulate how far \(\reg_{v'}\) is pushed from \(\reg_u\).

\section{On Parameterized Regions}
\subsection{General Parameterized Regions}\label{app:general_reg}
In this section, we extend previous results on Section \ref{sec:dep_dist} to general region case. 
A  set of \textbf{parameterized regions} over Euclidean Space $\mathbb{R}^n$ can be defined as all regions $\reg$ of the following form: 
\begin{equation}\label{eq:para_reg}
\reg(\bm{\theta}) = \{\x \in \mathbb{R}^n : f_i(\x, \bm{\theta}) \leq 0, \quad \text{for } i = 1, \ldots, n\},  
\end{equation}
where $f_1, \ldots, f_n$ are given (differentiable) functions from the given space to $\mathbb{R}$ and $\bm{\theta} = (\theta_1, \ldots, \theta_m)  \in \mathbb{R}^m$ is the $m$-dimensional parameter. 

Let \(\mathcal{R}\) denote a space of parameterized regions \(\reg(\bm{\theta})\) as defined in Equation~\ref{eq:para_reg}, with \(\bm{\theta} = (\theta_1, \ldots, \theta_m) \in \mathbb{R}^m\). 
For simplicity, and without significant loss of generality, in this section, we assume that all parameterized regions spaces \(\mathcal{R}\) satisfies the following properties:

\begin{itemize}[leftmargin=*]
  \item \textit{Uniqueness:} Two regions coincide if and only if their parameters are equal, i.e.,
  \(\reg(\bm{\theta}) = \reg(\bm{\theta}') \iff \bm{\theta} = \bm{\theta}'\).

  \item \textit{Volume-like parameter.}
The last coordinate \(\theta_m\) of \(\boldsymbol{\theta}\) serves as a volume parameter such that   \(\theta_m > 0\) and   \(
        \lim_{\reg(\boldsymbol{\theta})\rightarrow \emptyset} \theta_m = 0
    \)

% \begin{enumerate}
%     \item \textbf{Vanishing:}
%     \(
%         \lim_{\theta_m  \to 0} \operatorname{reg}(\boldsymbol{\theta})= \emptyset.
%     \)

%     \item \textbf{Monotonicity:}
%     If \(\theta_i' = \theta_i\) for all \(1 \le i < m\) and \(\theta_m' < \theta_m\), then
%     \[
%         \operatorname{reg}(\boldsymbol{\theta}') \subset \operatorname{reg}(\boldsymbol{\theta})
%         \quad\text{and}\quad
%         \operatorname{reg}(\boldsymbol{\theta}') \neq \operatorname{reg}(\boldsymbol{\theta}).
%     \]
% \end{enumerate}

  \item \textit{Non-empty regions:} For any \(\bm{\theta} \in \mathbb{R}^m\) with \(\theta_m > 0\), the corresponding region \(\reg(\bm{\theta})\) is non-empty.

\item \textit{Contractible:} Any region can be continuously shrunk to the empty set while remaining in the space $\mathcal{R}$. Formally, for any parameter vector $\boldsymbol{\theta}$, there exists a sequence $\{\boldsymbol{\theta}^k\}_{k=0}^{\infty}$ with $\boldsymbol{\theta}^0 = \boldsymbol{\theta}$ and $\theta_m^k > 0$ for all $k$, such that
\(
\reg(\boldsymbol{\theta}^{0}) \supset \reg(\boldsymbol{\theta}^{1}) \supset \reg(\boldsymbol{\theta}^{2}) \supset \cdots \to \emptyset.
\)
\end{itemize}

It is worth noting that the spaces of balls and boxes in Examples~\ref{exp:dist_depth_ball} and~\ref{exp:dist_depth_box} satisfy all of the above properties, with the volume parameter given by the radius $r$ or the offsets $\mathbf{o}$, respectively. 

Moreover, since we assume a positive volume parameter \( \theta_m > 0 \) under the ``volume-like parameter'' property, a single point usually does not qualify as a valid region. Consequently, entities such as nominals should be modeled as small regions rather than as points.

We extend Theorem \ref{theo:dep_dist} to parameterized regions above as followings. 

\begin{theo}
    Consider the parameterized region space \(\mathcal{R}\)  satisfies assumptions above, and equipped with the depth dissimilarity as defined in Definition~\ref{def:dep}. The following properties hold:
    \begin{enumerate}[leftmargin=*]
        \item For any \(\reg_1, \reg_2 \in \mathcal{R}\) and any \(\Delta > 0\), there exists \(\reg'\subseteq \reg_2\) such that
        \[
        \ddepth(\reg_1, \reg) > \ddepth(\reg_1, \reg_2) + \Delta, \quad \forall \reg \subseteq \reg'.
        \]

        % For any reg_1, reg_2 ∈ R and any Δ > 0, there exists reg' ⊆ reg_2 such that  d_dep(reg_1, reg) > d_dep(reg_1, reg_2) + Δ, for any reg ⊆ reg'.

        \item For any \(\reg \in \mathcal{R}\), any integer \(n\), and any \(M > 0\), there exist subsets \(\reg_1, \ldots, \reg_n \subseteq \reg\) such that for any distinct \(i, j \in \{1, \ldots, n\}\),
        \(
        \ddepth(\reg_i, \reg_j) > M.
        \)
    \end{enumerate}
\end{theo}
\begin{proof}[Outline of proof]
We sketch the main ideas, which closely follow the argument in the proof of Theorem~\ref{theo:dep_dist}.

For the first statement, by selecting a sufficiently small subset \(\reg' \subseteq \reg_2\), we can ensure that
\[
\|\para(\reg) - \para(\reg_1)\| > \delta
\]
for some \(\delta > 0\). Additionally, since \(f(\reg')\) can be made arbitrarily small as \(\reg' \to \emptyset\) (i.e., \(\lim_{\reg \to \emptyset} f(\reg) = 0\)), we can make the numerator of the depth dissimilarity arbitrarily small while keeping the denominator bounded below by a positive constant. Consequently, the depth dissimilarity \(\ddepth(\reg_1, \reg')\) can be made arbitrarily large.

For the second statement, a similar argument applies: by choosing the subsets \(\reg_1, \ldots, \reg_n\) to be sufficiently small and mutually separated within \(\reg\), we can ensure that each \(f(\reg_i)\) is small enough so that the depth dissimilarity between any pair \((\reg_i, \reg_j)\) exceeds \(M\).

\end{proof}
    
% For Proposition \ref{prop:dep_eq}, the mapping $F$ above can be generalized to arbitrary regions. For example, we could extend $F$ by replacing the center point and radius of balls with the center of gravity and diameters of any region. However, in this case, the function $F$ might be non-injective, as distinct regions may share the same center of gravity and diameters. Therefore, the resulting region space serves more as an extension rather than an equivalence of hyperbolic space.

Propositions~\ref{prop:dep_eq} and \ref{prop:depth_hyper} are extended to arbitrary regions as follows.

\begin{prop}\label{prop:general}
Let \(\mathbb{H}^{m}\) denote the hyperbolic space with curvature \(-1\). 
Assume that \(\mathcal{R}\) is a collection of parameterized regions satisfying the above assumptions, equipped with the depth dissimilarity defined in Equation~\eqref{eq:depth_dist}. 
Then the map \(F \colon \mathcal{R} \to \mathbb{H}^{m}\) defined by 
\(
F(\bm{\theta}) = \bm{\theta}
\) 
is a \emph{bijective isometry} between \(\mathcal{R}\) and \(\mathbb{H}^{m}\) when \(p = 2\), \(g(x) = \operatorname{arcosh}(x + 1)\), and \(f(\reg(\bm{\theta})) = \sqrt{2}\,\theta_m\). 
\end{prop}

\begin{proof}
Under the given condition,  the distance in the \(\mathbb{H}^{n+1}\) is given by:
\[
d_k(\x, \y) = \operatorname{arcosh}\left( 1 + \frac{\|\x - \y\|^2_2}{2 \x_n \y_n} \right).
\]
Then, the distance induced by the function \(F\) is of the form:
\[
d^\#_{\mathcal{R}}(\reg(\bm{\theta}), \reg(\bm{\theta'})) = \operatorname{arcosh}\left( 1 + \frac{\|\bm{\theta} - \bm{\theta}'\|^2_2 }{2 \theta_m \theta_m'} \right).
\]
This coincides with the depth dissimilarity in Equation \ref{eq:depth_dist} when \( p = 2 \), \( g(x) = \operatorname{arcosh}(x + 1) \), and \( f(\reg(\bm{\theta})) = \sqrt{2}\,\theta_m \).
\end{proof}

\begin{prop} Following Proposition \ref{prop:general},  let the depth dissimilarity $\ddepth$ be redefined by replacing \( g \) to \( g(x) = k\cdot x\, (k>0)\). Then the map \( F\) retains the following monotonicity property: for any points \(\x_1, \x_2, \x_3, \x_4 \in \mathbb{H}^n\),
\(
d_{\mathbb{H}^n}(\x_1, \x_2) < d_{\mathbb{H}^n}(\x_3, \x_4)
\)
if and only if
\[
\ddepth(F^{-1}(\x_1), F^{-1}(\x_2)) < \ddepth(F^{-1}(\x_3), F^{-1}(\x_4)).
\]
\end{prop}

\begin{proof}
    The proof is the same as the proof of Proposition \ref{prop:hyper} because of  the function \( h(x) = \operatorname{arcosh}(x + 1) \) is an increasing bijection from \(\mathbb{R}_{\geq 0}\) to \(\mathbb{R}_{\geq 0}\).
\end{proof}

\subsection{Faithfulness}
In this section, we show that embeddings based on arbitrary regions in the plane (i.e., $\mathbb{R}^2$) are \emph{faithful}, that is, they capture the hierarchy accurately, as detailed in Definition~\ref{def:faithful}. This result significantly strengthens existing theoretical guarantees. In particular, prior work shows that box embeddings can represent any DAG using $O((\Delta+2)\log n)$ dimensions, where $\Delta$ denotes the maximum degree and $n$ the number of nodes (Proposition~3 in~\cite{tBox}). In contrast, we prove that two dimensions already suffice when arbitrary regions are allowed.

\begin{definition}[Faithful embedding]\label{def:faithful}
Let $(V,\prec)$ be a partially ordered set.
An \emph{$n$-dimensional embedding} $f$ of $V$  that assigns to each element $v \in V$ a region $f(v) \subseteq \mathbb{R}^n$ is said to be \emph{faithful} if, for all $v\neq v' \in V$,
\[
v \prec v' \quad \Longleftrightarrow \quad f(v) \subseteq f(v').
\]
\end{definition}

\begin{figure}
    \centering
\begin{tikzpicture}
            % Polygon parameters
            \def\n{7}        % Number of vertices
            \def\radius{1.3} % Radius of the circumscribed circle
         % Draw grid (optional)
\draw[help lines, gray!30, step=0.1] (-1.1*\radius,-1.1*\radius) grid (1.1*\radius,1.1*\radius);
% Draw axes
% \draw[->,thick] (-1.1*\radius,0) -- (1.1*\radius,0) node[right] {};
% \draw[->,thick] (0,-1.1*\radius) -- (0,1.1*\radius) node[above] {};

            % Draw vertices
            \foreach \i in {1,...,\n} {
                \coordinate (v\i) at ({360/\n * (\i-1)}:\radius);
            }

            % Draw vertices as points and label them
            \foreach \i in {1,...,\n} {
                \filldraw (v\i) circle(2pt); % Draw the vertex as a small point
                \node[below left] at (v\i) {\small $w_\i$}; % Label the vertex
            }
            
            % Polyhedral angle visualization
            \begin{scope}[shift={(0,0,2)}]
                % Highlight the polyhedral angle surface
                \draw[fill=red!10,opacity=0.5] (v1) -- (v4) -- (v5)  -- (v6)-- cycle;
                \node[red!50] at (barycentric cs:v1=0.5,v4=1,v5=1,v6=0.5) {$f(v_1)$};
                
                \draw[fill=blue!10,opacity=0.5] (v2) -- (v6) -- (v7) -- cycle;
                \node[blue!50] at (barycentric cs:v2=0.5,v6=0.8,v7=0.8) {$f(v_2)$};
                
                 \draw[fill=black!10,opacity=0.5] (v3) -- (v5) -- (v7) -- cycle;
                 \node[black!50] at (barycentric cs:v3=1,v5=0.5,v7=0.4) {$f(v_3)$};
            \end{scope}

        \begin{scope}[shift={(-4.5,0,2)}]
             % Nodes
        \node (R1) at (-0.5, 2) {$v_1$};
        \node (R15) at (0.5, 2) {$v_2$};
        \node (R2) at (1.6, 2) {$v_3$};
        
        \node (L1) at (-1, 0) {$v_4$};
        \node (L2) at (0, 0) {$v_5$};
        \node (L3) at (1, 0) {$v_6$};
        \node (L4) at (2, 0) {$v_7$};
        % \node (L5) at (3, 0) {$v_5$};

        % Edges
    \draw[->] (R1) -- (L1);
    \draw[->] (R1) -- (L2);
    \draw[->] (R1) -- (L3);
    % \draw[->] (R2) -- (L3);
    % \draw[->] (R2) -- (L4);
    \draw[->] (R2) -- (L4);
     \draw[->] (R2) -- (L2);
    \draw[->] (R15) -- (L4);
    \draw[->] (R15) -- (L3);
         \end{scope}
        \end{tikzpicture}
\caption{Example of a partially ordered set (left, $v \rightarrow v'$ means $v \prec v'$) and its embedding (right), with $S_{v_1} = \{w_1, w_4, w_5, w_6\}$, $S_{v_2} = \{w_2, w_6, w_7\}$, and $S_{v_3} = \{w_3, w_5, w_7\}$.}
    \label{fig:exp_embed}
\end{figure}

\begin{theorem}
For any partially ordered set $(V,\prec)$, there exists a 2-dimensional faithful embedding of $V$. 
\end{theorem}

\begin{proof}
Let $V=\{v_1,\ldots,v_m\}$.
We construct a two-dimensional embedding as follows (see Figure~\ref{fig:exp_embed}):
\begin{enumerate}
    \item Embed the elements $v_i$ as distinct points
    $w_i$ placed at the vertices of a regular $m$-gon in $\mathbb{R}^2$.

    \item For each $v_i \in V$, let
    \(
    S_{v_i} :=\{w_i\} \cup \{\, w_j \mid v_j \prec v_i \,\}.
    \)
    We define $f(v_i)$ to be the convex hull of the points in $S_{v_i}$.
\end{enumerate}
The embedding is faithful as follows: For any $v_1 \neq v_2\in V$: 
\textbf{($\Rightarrow$)}
if $v_1 \prec v_2$, then by transitivity of the partial order,
$S_{v_1} \subseteq S_{v_2}$.
Since the convex hull operator is monotone with respect to set inclusion,
we obtain 
\(
f(v_1) \subseteq f(v_2).
\)
\textbf{($\Leftarrow$)}
if $f(v_1) \subseteq f(v_2)$.
By our construction based on convex hull of $n$-gons, $w_1 \in f(v_2)$ implies $w_1 \in S_{v_2}$.
By definition of $S_{v_2}$, this yields $v_1 \prec v_2$.

\end{proof}

The construction above embeds each node as a polyhedral region with a variable number of vertices, which is generally impractical for real-world applications. In practice, one usually learn embeddings  through a training procedure that restricts regions to a fixed parametric family as in Section \ref{sec:train}.

\section{Evaluation}\label{sec:eval}
Our experiments aim to address two questions: 1. How effectively do our methods capture hierarchical relationships? 2.  Can they generalize to tasks involving more than hierarchies?

% We evaluate hierarchical relationship modeling using transitive DAGs (Section~\ref{sec:DAG}) and \emph{ontologies}, where hierarchies are represented by the ``SubclassOf'' relation ($\sqsubseteq$). Specifically, we assess both the inference (Section~\ref{sec:Ont_inf}) and prediction (Section~\ref{sec:Ont_pred}) of axioms.

We evaluate two types of reasoning behavior. First, in Section~\ref{sec:DAG}, we examine transitivity reasoning over directed acyclic graphs (DAGs), focusing on transitive relations such as the “is-a” relation in the dataset: Noun. Second, in Sections~\ref{sec:Ont_inf} and \ref{sec:Ont_pred}, we study entailment reasoning over ontologies, specifically for inferring concept subsumptions. For both tasks, we use F1-score as the primary metric for logical reasoning performance, assessing the accuracy of inferred results under the embedding setting. 
% Due to space limitations, additional results and detailed experimental settings are provided in Appendix~\ref{app:exp}. 

%Further details are provided in Appendix \ref{app:ont}.
% \begin{itemize} 
% \item Do our methods effectively capture the hierarchy relations (i.e., DAG)? 
% \item Can our settings generalize to tasks beyond hierarchy-specific applications? 
% \end{itemize}

% In our evaluation, we not only assess how effectively our methods capture hierarchical data (i.e., DAGs), but we also explore their performance on tasks beyond hierarchy-specific applications, which is done by using ontology datasets.

% \begin{enumerate}
%     \item Ontologies can represent complex concepts constructed from basic elements using logical operators and relations. For instance, the concept of ``Mother'' can be formally defined as a female human who has at least one child, expressed in description logic as $\mathit{Mother} \equiv \mathit{Female} \sqcap \exists_{\mathit{hasChild}}\mathit{Human}$.

%     \item Ontologies support formal reasoning mechanisms that can derive new knowledge through complex logical patterns, extending beyond the simple transitive closure operations found in DAGs. As an example, given that $\mathit{Ann}$ is female and has a child $\mathit{Tom}$, an ontological reasoning system can infer that $\mathit{Ann}$ is a mother based on the previous definition---an inference impossible through transitive closure alone.
% \end{enumerate}

\begin{table*}[t]
\centering
\begin{tabular}{clcccccccc}
\toprule
\multicolumn{2}{c}{\multirow{2}{*}{\textbf{Method}}}& \multicolumn{2}{c}{\textbf{Mammal}} & \multicolumn{2}{c}{\textbf{Noun}} & \multicolumn{2}{c}{\textbf{MCG}} & \multicolumn{2}{c}{\textbf{Hearst}} \\
\cmidrule(lr){3-4} \cmidrule(lr){5-6} \cmidrule(lr){7-8} \cmidrule(lr){9-10}
&& d=2 & d=5 & d=5 & d=10 & d=5 & d=10 & d=5 & d=10 \\
\midrule
\multicolumn{2}{c}{$\tau$Box }& 29.0 & 33.5 & 30.5 & 31.5 & 43.9 & 50.3 & 39.7 & 43.7 \\
\multicolumn{2}{c}{OE }& 25.4& 31.0& 28.8& 30.8&36.3 & 46.6 & 34.6 & 40.7\\
\multicolumn{2}{c}{ELBE (box baseline)}& 30.3 & 36.8 & 30.7& 31.8&48.4& 55.5 & 41.6 & 46.8  \\
\multicolumn{2}{c}{ELEM (ball baseline)}& 27.7 & 28.8 & 28.6 & 29.5 & 35.7 & 38.6 & 34.6 & 36.7 \\
\hline
\multicolumn{2}{c}{EntailmentCone$^*$} & 54.4 & 56.3 & 29.2 & 32.1 & 25.3 & 25.5 & 22.6 & 23.7 \\
\multirow{2}{*}{ShadowCone$^*$}& (Umbral-half) & 57.7 & 69.4 & 45.2 & 52.2 & 36.8 & 40.1 & 32.8 & 32.6 \\
& (Penumbral-half) & 52.8 & 67.8 & 44.6 & 51.7 & 35.0& 37.6 &  26.8 & 28.4\\
\midrule
\multirow{2}{*}{RegD} & (box) & \textbf{64.9} & 71.6 & 53.8 & 51.3 & \textbf{50.7 }& \textbf{58.5 }&\textbf{ 42.8} & \textbf{49.6 }\\
%\cmidrule(lr){3-10}
& (ball) & 62.7&\textbf{71.8}& \textbf{58.4} & \textbf{59.1} & 44.9& 46.8 & 37.7 & 37.7\\

% \hline
% \multirow{2}{*}{RegD} & (box) & \textbf{63.00}($\pm$2.01) & \textbf{69.65}($\pm$3.14) & \textbf{52.98}($\pm$1.00) & 49.33($\pm$1.12) & \textbf{50.49}($\pm$0.44)& \textbf{57.71}($\pm$0.41)&\textbf{ 42.70}($\pm$0.50) & \textbf{48.72}($\pm$0.12)\\
% %\cmidrule(lr){3-10}
% & (ball) & 56.70($\pm$4.91)&66.76($\pm$2.91)& 51.63($\pm$1.52) & \textbf{59.31}($\pm$2.29) & 44.09($\pm$0.55)& 46.35($\pm$0.37) & 37.24($\pm$0.37) & 37.11($\pm$0.47)\\

% % \hline
% \multirow{2}{*}{RegD$^{**}$} & (box)  & \textbf{63.0} & \textbf{69.7} & \textbf{53.0} & 49.3 & \textbf{50.5}& \textbf{57.7}&\textbf{ 42.7} & \textbf{48.7}\\
% %\cmidrule(lr){3-10}
% & (ball) & 56.7&66.8& 51.6 & \textbf{59.3} & 44.1& 46.4 & 37.2 & 37.1\\

% \hline
% \multirow{2}{*}{RegD} & (box) & \textbf{$63.0_{(\pm2.0)}$} & \textbf{$69.7_{(\pm3.1)}$} & \textbf{$53.0_{(\pm1.0)}$} & $49.3_{(\pm1.1)}$ & \textbf{$50.5_{(\pm0.4)}$}& \textbf{$57.7_{(\pm0.4)}$}&\textbf{$42.7_{(\pm0.5)}$} & \textbf{$48.7_{(\pm0.1)}$}\\
% %\cmidrule(lr){3-10}
% & (ball) & $56.7_{(\pm4.9)}$&$66.8_{(\pm2.9)}$& $51.6_{(\pm1.5)}$ & \textbf{$59.3_{(\pm2.3)}$} & $44.1_{(\pm0.6)}$& $46.4_{(\pm0.4)}$ & $37.2_{(\pm0.4)}$ & $37.1_{(\pm0.5)}$\\
\bottomrule
\end{tabular}%
\caption{F1 score (\%) on Mammal, WordNet noun, MCG, and Hearst.
Results with * are coming from ShadowCones.} 
\label{table:f1_scores}
\end{table*}

\subsection{Inferences over DAG}\label{sec:DAG}

\subsubsection{Benchmark} Following \cite{ShadowCone}, we evaluate our method on four real-world datasets consisting of Is-A relations: MCG \cite{DBLP:conf/cikm/WangWWX15,DBLP:conf/sigmod/WuLWZ12}, Hearst patterns \cite{DBLP:conf/coling/Hearst92}, the WordNet \cite{fellbaum1998wordnet} noun taxonomy, and its mammal subgraph. All models are trained exclusively on \textit{basic edges}, which are edges not implied transitively by other edges in the graph. For validation and testing, we use the same sets as in \cite{ShadowCone}, consisting of 5\% of \textit{non-basic (inferred)} edges, ensuring a fair comparison.  
% {\color{red}The hyperparameter settings are provided in Appendix \ref{app:hyperparameter}.} 
We exclude non-basic edges from training since they can be transitively derived from basic edges. Including them would artificially inflate performance metrics without properly evaluating the embeddings' ability to capture hierarchical structures. 
% {\color{red}For completeness, results for non-basic cases are provided in Appendix \ref{app:other_result}}. 

% We exclude non-basic edges from the training set since they are transitively derivable from basic edges. Including them would artificially strengthen the training data, inflating performance metrics without accurately reflecting the ability of embedding methods to capture hierarchical structures. Focusing on basic edges ensures the evaluation measures the embeddings' true representational capabilities rather than the effects of data augmentation.

% \paragraph{Benchmark} We follow the approach in \cite{ShadowCone} and use four hierarchical relations across four DAGs based on the Is-A relations: MCG \cite{DBLP:conf/cikm/WangWWX15,DBLP:conf/sigmod/WuLWZ12}, Hearst patterns \cite{DBLP:conf/coling/Hearst92}, the WordNet \cite{fellbaum1998wordnet} Noun taxonomy, and its Mammal sub-graph. We train our model on the basic edges\footnote{i.e., edges $(u, v)$ such that no other edges of the form $(u, v')$ and $(v', u)$ exist in the graph} and use the same validation and test sets consist of 5\% of non-basic (inferred) edges as in \cite{ShadowCone} to ensure a fair comparison. 

\subsubsection{Baselines} We compare our method RegD with (i) hyperbolic approaches such as EntailmentCone \cite{EntCone} and ShadowCone \cite{ShadowCone}, which is the latest method with the state-of-the-art performance; (ii) region-based methods like OrderEmbedding \cite{DBLP:conf/nips/BordesUGWY13}, %GBC- and VBC-box \cite{DBLP:conf/nips/ZhangBMM22}, 
and $\tau$Box \cite{tBox}. We also compare with the ontology embedding methods, ELBE \cite{peng_description_2022} and ELEM \cite{kulmanov_embeddings_2019}, which can be considered as the baseline approaches embedding the DAG as boxes or balls, respectively. We use F1-scores  as in previous studies.

It is worth note that some prior works, such as RotL and Rot2L~\cite{DBLP:conf/emnlp/WangLLS21}, also simplify operations in hyperbolic space. However, they focus on M\"obius addition and address link prediction in knowledge graphs. Therefore, we do not compare with these methods.

\subsubsection{Results}  
The performance comparison across different DAGs is shown in Table \ref{table:f1_scores}. RegD achieved the best performance on all four datasets. Notably, the box variant consistently outperformed the ball variant in most cases, which might be because boxes contain more parameters than balls when embedded in the same dimensional space.  
Interestingly, region-based methods outperformed hyperbolic methods on the MCG and Hearst datasets. However, on the Noun and mammal dataset, hyperbolic methods performed better.  Nevertheless, our method performed consistently well in both cases, as it can adjust the hyperbolic component by setting different \(\lambda\) values in Equation (\ref{eq:energy}). 

\begin{table}[t]
    \centering
      \begin{tabular}{ccccccccc}
          \toprule
          \multirow{2}{*}{\textbf{Method}} & \multicolumn{2}{c}{\textbf{GALEN}} & \multicolumn{2}{c}{\textbf{GO}} & \multicolumn{2}{c}{\textbf{ANATOMY}} \\
          \cmidrule(lr){2-3} \cmidrule(lr){4-5} \cmidrule(lr){6-7}
          & d=5 & d=10 & d=5 & d=10 & d=5 & d=10 \\
          \midrule
             ELBE & 20.7 & 21.2 & 36.9 & 42.4 & 43.1 & 43.0 \\
             + $\tau$Box & 20.8 & 20.7 & 32.2 & 34.7 & 42.2 & 47.2 \\
             + RegD & \textbf{25.2} & \textbf{25.8} & \textbf{50.0} & \textbf{50.5} & \textbf{58.7} & \textbf{62.5} \\
          \midrule
             ELEM & 16.9 & 17.3 & 23.5 & 27.4 & 34.6 & 38.7 \\
             + RegD & \textbf{19.2} & \textbf{18.8} & \textbf{36.4} & \textbf{40.0} & \textbf{52.5} & \textbf{55.5} \\
          \bottomrule
      \end{tabular}
              \caption{F1 score (\%) for the inference task.}
        \label{table:inference_ont}
\end{table}

\begin{table}[t]
        \centering
      \begin{tabular}{ccccccc}
          \toprule
          \multirow{2}{*}{\textbf{Method}} & \multicolumn{2}{c}{\textbf{GALEN}} & \multicolumn{2}{c}{\textbf{GO}} & \multicolumn{2}{c}{\textbf{ANATOMY}} \\
          \cmidrule(lr){2-3} \cmidrule(lr){4-5} \cmidrule(lr){6-7}
          &d=5 & d=10 & d=5 & d=10 & d=5 & d=10 \\
          \midrule
          ELBE & \textbf{21.0} & \textbf{24.2} & 32.8 & 37.9 & 25.1 & 25.5 \\
          + $\tau$Box & 18.4 & 19.5 & 24.0 & 29.3 & 25.6 & 22.9 \\
          + RegD & 20.6 & 21.0 & \textbf{37.1} & \textbf{44.1} & \textbf{41.4} & \textbf{45.3} \\
          \midrule
          ELEM & 16.7 & 16.6 & 54.5 & 54.2 & \textbf{23.1} & \textbf{26.0} \\
          + RegD & \textbf{16.8} & \textbf{18.0} & \textbf{60.3} & \textbf{61.4} & 21.7 & 21.9 \\
          \bottomrule
      \end{tabular}
       \caption{F1 score (\%) for the prediction task.}
        \label{table:predict_ont}
\end{table}

\subsection{Inference and Prediction over Ontologies}\label{sec:Ont_inf}

\subsubsection{Benchmark} We utilize three normalized biomedical ontologies: GALEN \cite{rector1996galen}, Gene Ontology (GO) \cite{ashburner2000gene}, and Anatomy (Uberon) \cite{mungall2012uberon}. As in \cite{DBLP:conf/www/JackermeierC024},  we use the entire ontology for training, and the complete set of inferred class subsumptions for testing. Those subsumptions can be regarded as partial order pairs $u\prec v$. Evaluation is performed using 1,000 subsumptions randomly sampled from the test set. Similar to inference over DAG, negative samples are generated by randomly replacing the child of each positive pair 10 times.

\subsubsection{Baselines}
We consider two representative geometric, model-based ontology embedding methods: ELBE~\cite{peng_description_2022} and ELEM~\cite{kulmanov_embeddings_2019}, which embed ontology concepts as balls and boxes, respectively. Those methods can be enhanced by RegD (or $\tau$Box in box case) by incorporating their corresponding energy functions to model inclusion constraints between geometric objects. 
Other hierarchy embedding methods are excluded from our evaluation due to their incompatibility with ontology embedding frameworks. For instance, OE represent concepts using cones, which cannot be directly integrated with ELBE or ELEM for ontology-specific reasoning tasks.

\subsubsection{Results}
The results are summarized in Tables \ref{table:inference_ont}. We can see that RegD yields consistent improvements across all ontologies for inference tasks. Specifically, it gains an F1 score increase of more than 45\% with ELBE method on the ANATOMY ontology.  Conversely, the $\tau$Box plugin consistently reduces performance across nearly all test cases, underscoring its limited applicability to tasks involving more than hierarchies.

\subsection{Prediction over Ontologies}\label{sec:Ont_pred}

We use the same baselines and datasets as described in Section \ref{sec:Ont_inf}. However, in this prediction task,  we partition the original ontologies directly into 80\% for training, 10\% for validation, and 10\% for testing as in \cite{DBLP:conf/www/JackermeierC024}. For the link prediction task, we focus on specific parts of the validation and testing sets, represented as $\exists r. B \prec ?A$, where $A$ and $B$ are concept names and $r$ is a role. This setup is equivalent to link prediction tasks $(?A, r, B)$ in knowledge graphs if we regard $A$, $B$, and $r$ as the head entity, tail entity, and relation, respectively.

\subsubsection{Results}
Table \ref{table:inference_ont} summarizes the results. RegD shows mixed results: while it generally improves performance, it decreases scores on the GALEN ontology and ANATOMY ontology with ELEM. This degradation likely occurs in challenging prediction cases where all methods perform poorly. Nevertheless, RegD achieves significant improvements in other cases, notably increasing F1 score by 77.6\% with ELBE on the ANATOMY ontology. In contrast, the $\tau$Box plugin consistently reduces performance across almost all test cases.

\section{Conclusion}
We introduced a framework RegD for low-dimensional embeddings of hierarchies, leveraging two dissimilarity metrics between regions. Our method, applicable to regions in the Euclidean space, demonstrates versatility and has the potential for a wide range of tasks involving data beyond hierarchies. Additionally, we showed that our approach achieves comparable embedding performance to hyperbolic methods while being significantly simpler to implement.

For future work, we are interested in  generalizing our framework to support a wider variety of region types beyond the balls and boxes considered here, so that it can be better adapted to the requirements of diverse downstream tasks. 
Another promising direction for future work is to investigate replacing existing hyperbolic embedding components in related methods with our approach, including recent studies that couple large language models with hyperbolic embeddings to learn semantic hierarchies~\cite{he2024language}, which can be extended from atomic concept to complex ones by integrating the ontology embeddings as in Section \ref{sec:Ont_inf}.

% There remains a lack of formal and rigorous explanation for the suboptimal performance of hyperbolic-based methods on the mammal and noun datasets. Integration with more advanced ontology embedding approaches also remains unexplored. 
% In future work, we aim to address these limitations by expanding the range of benchmarks and providing a more thorough theoretical analysis of performance differences. 
% Additionally, we plan to extend our framework to incorporate semantic information using language models, following approaches such as \cite{he2024language}.

% \paragraph{Reproducibility statement} All code and data are publicly available at \url{https://anonymous.4open.science/r/RegD-F4E3}, along with clear instructions on environment configuration and hyperparameter settings to enable full reproduction of our results.

% \section*{AI declaration}

% During the preparation of this work, the authors used AI to improve the work's readability and language. After using this tool, the authors reviewed and edited the content as needed and take full responsibility for the content of the publication.

\section*{ACKNOWLEDGEMENT}
This work is funded by EPSRC projects OntoEm (EP/Y017706/1).

\bibliographystyle{kr}
\bibliography{main}
\end{document}

%% file: math_commands.tex
\usepackage{amsmath,amsfonts,bm}

\def\eqref#1{equation~\ref{#1}}
\def\1{\bm{1}}

\DeclareMathAlphabet{\mathsfit}{\encodingdefault}{\sfdefault}{m}{sl}
\SetMathAlphabet{\mathsfit}{bold}{\encodingdefault}{\sfdefault}{bx}{n}

%% file: marcos.tex
\newcommand{\R}{\ensuremath{\mathbb{R}}\xspace}
\newcommand{\Sp}{\ensuremath{\mathbb{S}}\xspace}
\newcommand{\M}{\ensuremath{\mathcal{M}}\xspace}
\newcommand{\C}{\ensuremath{\mathcal{C}}\xspace}
\newcommand{\N}{\ensuremath{\widetilde{\mathcal{N}}}\xspace}
\newcommand{\Nstar}{\ensuremath{\mathcal{N}^*}\xspace}
\newcommand{\Q}{\ensuremath{\mathcal{Q}}\xspace}

\newcommand{\x}{\ensuremath{\mathbf{x}}\xspace}
\newcommand{\bc}{\ensuremath{\mathbf{c}}\xspace}
\newcommand{\bo}{\ensuremath{\mathbf{o}}\xspace}
\newcommand{\bR}{\ensuremath{\mathbf{R}}\xspace}
\newcommand{\bB}{\ensuremath{\mathbf{B}}\xspace}
\newcommand{\bK}{\ensuremath{\mathbf{K}}\xspace}
\newcommand{\ba}{\ensuremath{\mathbf{a}}\xspace}
\newcommand{\bP}{\ensuremath{\mathbf{P}}\xspace}
\newcommand{\Op}{\ensuremath{\mathbf{O}}\xspace}
\newcommand{\y}{\ensuremath{\mathbf{y}}\xspace}
\newcommand{\V}{\ensuremath{\mathbf{V}}\xspace}
\newcommand{\z}{\ensuremath{\mathbf{z}}\xspace}
\newcommand{\A}{\ensuremath{\mathbf{A}}\xspace}
\newcommand{\X}{\ensuremath{\mathbf{X}}\xspace}
\newcommand{\I}{\ensuremath{\mathbf{I}}\xspace}
\newcommand{\D}{\ensuremath{\mathbf{D}}\xspace}
\newcommand{\bL}{\ensuremath{\mathbf{L}}\xspace}
\newcommand{\Lmc}{\ensuremath{\mathcal{L}}\xspace}
\newcommand{\W}{\ensuremath{\mathbf{W}}\xspace}
\newcommand{\bH}{\ensuremath{\mathbf{H}}\xspace}
\newcommand{\Wm}{\ensuremath{\mathbf{U}}\xspace}
\newcommand{\Ori}{\ensuremath{\mathbf{o}}\xspace}
\newcommand{\tran}{\ensuremath{\mathit{\tran}}}
\newcommand{\Bbox}{\ensuremath{\mathit{Box}}}
\newcommand{\Vol}{\ensuremath{\mathit{Vol}}}
\newcommand{\Regions}{\ensuremath{\mathcal{R}}}
\newcommand{\vol}{\mathrm{vol}}
\newcommand{\reg}{\textit{reg}}
\newcommand{\Ball}{\textit{ball}}
\newcommand{\boxes}{\textit{box}}
\newcommand{\para}{\textit{P}}
\newcommand{\dboundary}{\ensuremath{d_{\text{bd}}}}
\newcommand{\ddepth}{\ensuremath{d_{\text{dep}}}}

\newcommand{\rot}{\ensuremath{\textit{Rot}}\xspace}

\newcommand{\tA}{\ensuremath{\Tilde{\mathbf{A}}}\xspace}
\newcommand{\tD}{\ensuremath{\Tilde{\mathbf{D}}}\xspace}
\newcommand{\bu}{\ensuremath{\mathbf{u}}\xspace}
\newcommand{\bv}{\ensuremath{\mathbf{v}}\xspace}
\newcommand{\be}{\ensuremath{\mathbf{e}}\xspace}
\newcommand{\bp}{\ensuremath{\mathbf{p}}\xspace}

\newcommand{\Bmc}{\ensuremath{\mathcal{B}}\xspace}
\newcommand{\Omc}{\ensuremath{\mathcal{O}}\xspace}
\newcommand{\ORI}{\ensuremath{\mathcal{O}_{RI}}\xspace}
\newcommand{\OCI}{\ensuremath{\mathcal{O}_{CI}}\xspace}
\newcommand{\Imc}{\ensuremath{\mathcal{I}}\xspace}
\newcommand{\Jmc}{\ensuremath{\mathcal{J}}\xspace}
\newcommand{\Xmc}{\ensuremath{\mathcal{X}}\xspace}
\newcommand{\Mmc}{\ensuremath{\mathcal{M}}\xspace}
\newcommand{\starM}{$\top\!\bot^\ast$-module}
\newcommand{\ALC}{\ensuremath{\mathcal{ALC}}\xspace}
\newcommand{\ALCH}{\ensuremath{\mathcal{ALCH}}\xspace}
\newcommand{\EL}{\ensuremath{\mathcal{EL}}\xspace}
\newcommand{\ELp}{\ensuremath{\mathcal{EL}^+}\xspace}
\newcommand{\ELH}{\ensuremath{\mathcal{ELH}}\xspace}

\newcommand{\NC}{\ensuremath{\mathsf{N_C}}\xspace}
\newcommand{\NI}{\ensuremath{\mathsf{N_I}}\xspace}
\newcommand{\NR}{\ensuremath{\mathsf{N_R}}\xspace}
\newcommand{\ND}{\ensuremath{\mathsf{N_D}}\xspace}

\newcommand{\bmu}{\boldsymbol{\mu}}
\newcommand{\boxsqel}{Box$^2$EL\xspace}